\documentclass[hidelinks,11pt]{article}

\usepackage{psfrag,ifthen,multirow,mathrsfs}
\usepackage{epsfig,psfrag,float,texnansi,color,tikz,sectsty,ifthen}
\usepackage{algorithm}
\usepackage{algpseudocode}
\usepackage{microtype}
\usepackage{graphicx}
\usepackage{latexsym,url,afterpage,enumerate,lmodern}
\usepackage[hypertexnames=false,hyperfootnotes=false]{hyperref}
\usepackage[normalem]{ulem}
\usepackage{amsmath,amsthm,amssymb,amsfonts}
\usepackage[margin=10pt,font=small,labelfont=bf]{caption}
\usepackage[sort&compress]{natbib}
\usepackage{natbib}
\usepackage{booktabs,xcolor,amsmath}
\usepackage{IEEEtrantools}
\usepackage{bm}
\usepackage{enumitem}
\usepackage{tabularx}
\usepackage{dsfont}
\usepackage{rotating}
\usepackage[caption=false]{subfig}

\newcommand{\sets}[1]{\ensuremath{\mathcal{#1}}}

\newtheoremstyle{thm-sf}{}{}{\itshape}{}{\sffamily\bfseries}{.}{ }{}
\theoremstyle{thm-sf}

\newtheorem{example}{Example}
\newtheorem{theorem}{Theorem}

\newtheorem{proposition}{Proposition}

\usetikzlibrary{arrows,patterns,plotmarks}
\tikzstyle{every picture} += [>=stealth]
\allsectionsfont{\sffamily}
\makeatletter
\def\@seccntformat#1{\csname the#1\endcsname.\quad}
\makeatother

\usepackage{etoolbox}

\patchcmd\subequations
 {\theparentequation\alph{equation}}
 {\subequationsformat}
 {}{}

\newcommand{\subequationsformat}{\theparentequation.\arabic{equation}}

\floatstyle{ruled}
\provideboolean{fastcompile}
\newcommand{\hidefastcompile}[1]{\ifthenelse{\boolean{fastcompile}}{}{#1}}

\newcommand{\newpv}[1]{{\color{black} #1}}

\newcommand{\rev}[1]{{\color{black} #1}}

\usepackage{setspace}
\usepackage[total={6.5in,9in},top=1in, left=1in,]{geometry} 
\setcitestyle{numbers,square}

\usepackage{natbib}
 \bibpunct[, ]{(}{)}{,}{a}{}{,}%

\title{\textbf{\textsf{Learning Optimal Classification Trees:\\ Strong Max-Flow Formulations}}}
\author{Sina Aghaei$^1$, Andr\'es G\'omez$^2$, Phebe Vayanos$^1$ \vspace{0.6cm}\\ 
        \small{$^2$Department of Industrial and Systems Engineering, Viterbi School of Engineering} \\
        \small{$^1$Center for Artificial Intelligence in Society} \\
        \small{$^{1,2}$University of Southern California} \vspace{0.6cm} \\
        \texttt{\{saghaei,gomezand,phebe.vayanos\}@usc.edu}}%
 
\begin{document}

\date{}
\maketitle

\begin{abstract}
We consider the problem of learning optimal binary classification trees. Literature on the topic has burgeoned in recent years, motivated both by the empirical suboptimality of heuristic approaches and the tremendous improvements in mixed-integer programming (MIP) technology. Yet, existing approaches from the literature do not leverage the power of MIP to its full extent. Indeed, they rely on \emph{weak} formulations, resulting in slow convergence and large optimality gaps. To fill this gap in the literature, we propose a flow-based MIP formulation for optimal binary classification trees that has a \emph{stronger} linear programming relaxation. Our formulation presents an attractive decomposable structure. We exploit this structure and max-flow/min-cut duality to derive a Benders' decomposition method, which scales to larger instances. We conduct extensive computational experiments on standard benchmark datasets on which we show that our proposed approaches are 50 times faster than state-of-the art MIP-based techniques and  improve out of sample performance up to 13.8\%.
\end{abstract}

\section{Introduction}
\label{sec:Introduction}

\subsection{Motivation \& Related Work}

Since their inception over 30 years ago, decision trees have become among the most popular techniques for interpretable machine learning (classification and regression), see~\citet{breiman1984classification}. A decision tree takes the form of a \emph{binary tree}. In each \emph{internal} node of the tree, a binary test is performed on a specific feature. Two branches emanate from each internal node, with each branch representing the outcome of the test. If a datapoint passes (resp.\ fails) the test, it is directed to the left (resp.\ right) branch.  A predicted label is assigned to all \emph{leaf} nodes. Thus, each path from root to leaf represents a classification rule that assigns a unique label to all datapoints that reach that leaf. The goal in the design of optimal decision trees is to select the tests to perform at each internal node and the labels to assign to each leaf to maximize prediction accuracy (classification) or to minimize prediction error (regression). Not only are decision trees popular in their own right; they also form the backbone for more sophisticated machine learning models. For example, they are the building blocks for random forests, one of the most popular and stable machine learning techniques available, see e.g.,~\citet{liaw2002classification}. They have also proved useful to provide explanations for the solutions to optimization problems, see e.g.,~\citet{bertsimas2018voice}.

The problem of learning optimal decision trees is an $\mathcal{NP}$-hard problem, see~\citet{hyafil1976constructing} and~\citet{breiman2017classification}. It can intuitively be viewed as a combinatorial optimization problem with an exponential number of decision variables: at each internal node of the tree, one can select what feature to branch on (and potentially the level of that feature), guiding each datapoint to the left or right using logical constraints.

\paragraph{Traditional Methods.} Motivated by these hardness results, traditional algorithms for learning decision trees have relied on heuristics that employ very intuitive, yet ad-hoc, rules for constructing the decision trees. For example, CART uses the Gini Index to decide on the splitting, see~\citet{breiman1984classification}; ID3 employs entropy, see~\citet{quinlan1986induction}; and C4.5 leverages normalized information gain, see~\citet{quinlan2014c4}. The high quality and speed of these algorithms combined with the availability of software packages in many popular languages such as \texttt{R} or \texttt{Python} has facilitated their popularization, see e.g.,~\citet{therneau2015package,kuhn2018package}. They are now routinely used in commercial, medical, and other applications.

\paragraph{Mathematical Programming Techniques.} Motivated by the heuristic nature of traditional approaches, which provide no guarantees on the quality of the learned tree, several researchers have proposed algorithms for learning provably \emph{optimal} trees based on techniques from mathematical optimization. Approaches for learning optimal decision-trees rely on enumeration coupled with rules to prune-out the search space. For example, \citet{nijssen2010optimal} use itemset mining algorithms and~\citet{narodytska2018learning} use satisfiability (SAT) solvers. \citet{verhaeghe2019learning} propose a more elaborate implementation combining several ideas from the literature, including branch-and-bound, itemset mining techniques and caching. \citet{hu2019optimal} use analytical bounds (to aggressively prune-out the search space) combined with a tailored bit-vector based implementation.

\paragraph{The Special Case of MIP.} As an alternative approach to conducting the search for optimal trees, \citet{bertsimas2017optimal} recently proposed to use mixed-integer programming (MIP) to learn optimal classification trees. Following this work, using MIP to learn decision trees gained a lot of traction in the literature with the works of~\citet{gunluk2018optimal}, \citet{aghaei2019learning}, and~\citet{verwer2019learning}. This is no coincidence. First, MIP comes with a suit of off-the shelf solvers and algorithms that can be leveraged to effectively prune-out the search space. Indeed, solvers such as \citet{cplex2009v12} and \citet{gurobi2015gurobi}  have benefited from decades of research, see~\citet{bixby2012brief}, and have been very successful at solving a broad class of MIP problems. Second, MIP comes with a highly expressive language that can be used to tailor the objective function of the problem or to augment the learning problem with constraints of practical interest. For example, \citet{aghaei2019learning} leverage the power of MIP to incorporate fairness and interpretability constraints into learned classification and regression trees. They also show how MIP technology can be exploited to learn decision trees with more sophisticated structure (linear branching and leafing rules). Similarly, \citet{gunluk2018optimal} use MIP to solve classification trees with combinatorial branching decisions.
MIP formulations have also been leveraged to design decision trees for decision- and policy-making problems, see~\citet{azizi2018designing} and \citet{ciocan2018interpretable}, and for optimizing decisions over tree ensembles, see~\citet{mivsic2017optimization}.

\paragraph{Discussion \& Motivation.} The works of \citet{bertsimas2017optimal}, \citet{gunluk2018optimal}, \citet{aghaei2019learning}, and \citet{verwer2019learning} have served to showcase the modeling power of using MIP to learn decision trees and the potential suboptimality of traditional algorithms. Yet, we argue that they have not leveraged the power of MIP to its full extent. A critical component for efficiently solving MIPs is to pose good formulations, but determining such formulations is no simple task.  
The standard approach for solving MIP problems is the branch-and-bound method, which partitions the search space recursively and solves Linear Programming (LP) relaxations for each partition to produce lower bounds for fathoming sections of the search space. Thus, since solving a MIP requires solving a large sequence of LPs, small and compact formulations are desirable as they enable the LP relaxation to be solved faster. Moreover, formulations with tight LP relaxations, referred to as \emph{strong} formulations, are also desirable, as they produce higher quality lower bounds which lead to a faster pruning of the search space, ultimately reducing the number of LPs to be solved. Unfortunately, these two goals are at odds with one another, with stronger relaxations often requiring additional variables and constraints than \emph{weak} ones. For example, in the context of decision trees, \citet{verwer2019learning} propose a MIP formulation with significantly fewer variables and constraints than the formulation of \citet{bertsimas2017optimal}, but in the process weaken the LP relaxation. As a consequence, neither method consistently dominates the other. 

We note that in the case of MIPs with large numbers of decision variables and constraints, classical decomposition techniques from the Operations Research literature may be leveraged to break the problem up into multiple tractable subproblems of benign complexity. A notable example of a decomposition algorithm is  Benders' \citep{benders1962partitioning}. Bender's decomposition exploits the structure of mathematical programming problems with so-called \emph{complicating variables} which couple constraints with one another and which, once fixed, result in an attractive decomposable structure that is leveraged to speed-up computation and alleviate memory consumption, allowing the solution of large-scale MIPs. To the best of our knowledge, existing approaches from the literature have not sought explicitly strong formulations, neither have they attempted to leverage the potentially decomposable structure of the problem. This is precisely the gap we fill with the present work.

\subsection{Proposed Approach \& Contributions}

Our approach and main contributions in this paper are:
\begin{enumerate}[label=\emph{(\alph*)}]\setlength\itemsep{0em}
    \item We propose an intuitive flow-based MIP formulation \newpv{for learning optimal} classification trees with binary data. Notably, our proposed formulation does not use big-$M$ constraints, which are known to lead to weak LP relaxations. We also show that the resulting LP relaxation is stronger than existing alternatives.
    \item Our proposed formulation is amenable to Bender's decomposition. In particular, binary tests are selected in the master problem and each subproblem guides each datapoint through the tree via a max-flow formulation. We leverage the max-flow structure of the subproblems to solve them efficiently via min-cut procedures. 
    \rev{\item We present the first polyhedral results concerning the convex hull of the feasible region of decision trees: we show that all cuts added in our proposed Benders \newpv{method are \emph{facets}} of this decision tree polytope.}
    \item We conduct extensive \newpv{computational studies, showing} that our formulations improve upon the state-of-the-art MIP algorithms, both in terms of in-sample solution quality (and speed) and out-of-sample performance. 
\end{enumerate}
The proposed modeling and solution paradigm can act as a building block for the faster and more accurate learning of more sophisticated trees. Continuous data can be discretized and binarized to address problems with continuous labels, see \citet{breiman2017classification}. Regression trees can be obtained via minor modifications of the formulation, see e.g., \citet{verwer2017learning}. Fairness and interpretability constraints can naturally be incorporated into the problem, see~\citet{aghaei2019learning}. We leave these studies to future work.

The rest of the paper is organized as follows. We introduce our flow-based formulation and our Bender's decomposition method in \S\ref{sec:DT_Formulation} and \S\ref{sec:Bender}, respectively. We report in \S\ref{sec:Experiments} computational experiments with popular benchmark datasets. 

\section{Decision Tree Formulation}
\label{sec:DT_Formulation}

\subsection{Problem Formulation}

We are given a training dataset $\mathcal T:=\{ {\bm x}^i, y^i \}_{i\in \mathcal I}$ consisting of datapoints indexed in the set $\sets I$. Each row $i\in \sets I$ of this dataset consists of $F$ binary features indexed in the set $\sets F$ and collected in the vector ${\bm x}^i \in \{0,1\}^F$ and a label $y^i$ drawn from the finite set $\sets K$ of classes. We consider the problem of designing an optimal decision tree that minimizes the misclassification rate based on MIP technology.

The key idea behind our model is to augment the decision tree with a single source node~$s$ that is connected to the root node (node 1) of the tree and a single sink node~$t$ connected to all nodes of the tree, see Figure~\ref{fig:sample_tree}. This modification enables us to think of the decision tree as a \emph{directed acyclic graph with a single source and sink node}. Datapoints \emph{flow} from source to sink through a single path and only reach the sink if they are correctly classified (they will face a ``road block'' if incorrectly classified which will prevent the datapoint from traversing the graph at all). Similar to traditional algorithms for learning decision trees, we allow labels to be assigned to internal nodes of the tree. In that case, correctly classified datapoints that reach such nodes are directly routed to the sink node (as if we had a ``short circuit'').

Next, we introduce our notation and conventions that will be useful to present our model. We denote by~$\sets N$ and~$\mathcal L$ the sets of all internal and leaf nodes in the tree, respectively. For each node $n \in \sets N \cup \sets L$, we let $a(n)$ be the direct ancestor of $n$ in the graph. For $n\in \sets N$, let $\ell(n)$ (resp.\ $r(n)$) $\in  \mathcal N\cup \sets L$ represent the left (resp.\ right) direct descendant of node~$n$ in the graph. In particular, we have $a(1)=s$. We will say that we \emph{branch on feature $f\in \sets F$ at node $n\in \sets N$} if the binary test performed at~$n$ asks ``Is $x^i_f=0$''? Datapoint~$i$ will be directed left (right) if the answer is affirmative (negative).

The decision variables for our formulation are as follows.
The variable $b_{nf} \in \{0,1\}$ indicates if we branch on (i.e., perform a binary test on) feature $f \in \sets F$ at node $n \in \sets N$. If $\sum_{f\in F} b_{nf}=0$ for some node $n \in \sets N$, no feature is selected to branch on at that node, and a class is assigned to node~$n$. We let the variable $w_{nk} \in \{0,1\}$ indicate if the predicted class for node $n \in \sets N \cup \sets L$ is $k \in \sets K$. A datapoint $i$ is correctly classified iff it reaches some node $n$ such that $w_{nk}=1$ with $ k=y^i$. Points that arrive at that node and that are correctly classified are directed to the sink. 
For each node $n \in \sets N$ and for each datapoint $i \in \sets I$, we introduce a binary valued decision variable $z^i_{a(n),n}$ which equals 1 if and only if the $i$th datapoint is correctly classified (i.e., reaches the sink) and traverses the edge between nodes $a(n)$ and $n$. We let $z^i_{n,t}$
be defined accordingly for each edge between node $n \in \sets N \cup \sets L$ and sink $t$.


\begin{figure}[t]
\vskip 0.2in
\begin{center}
\centerline{\includegraphics[width=0.4\textwidth]{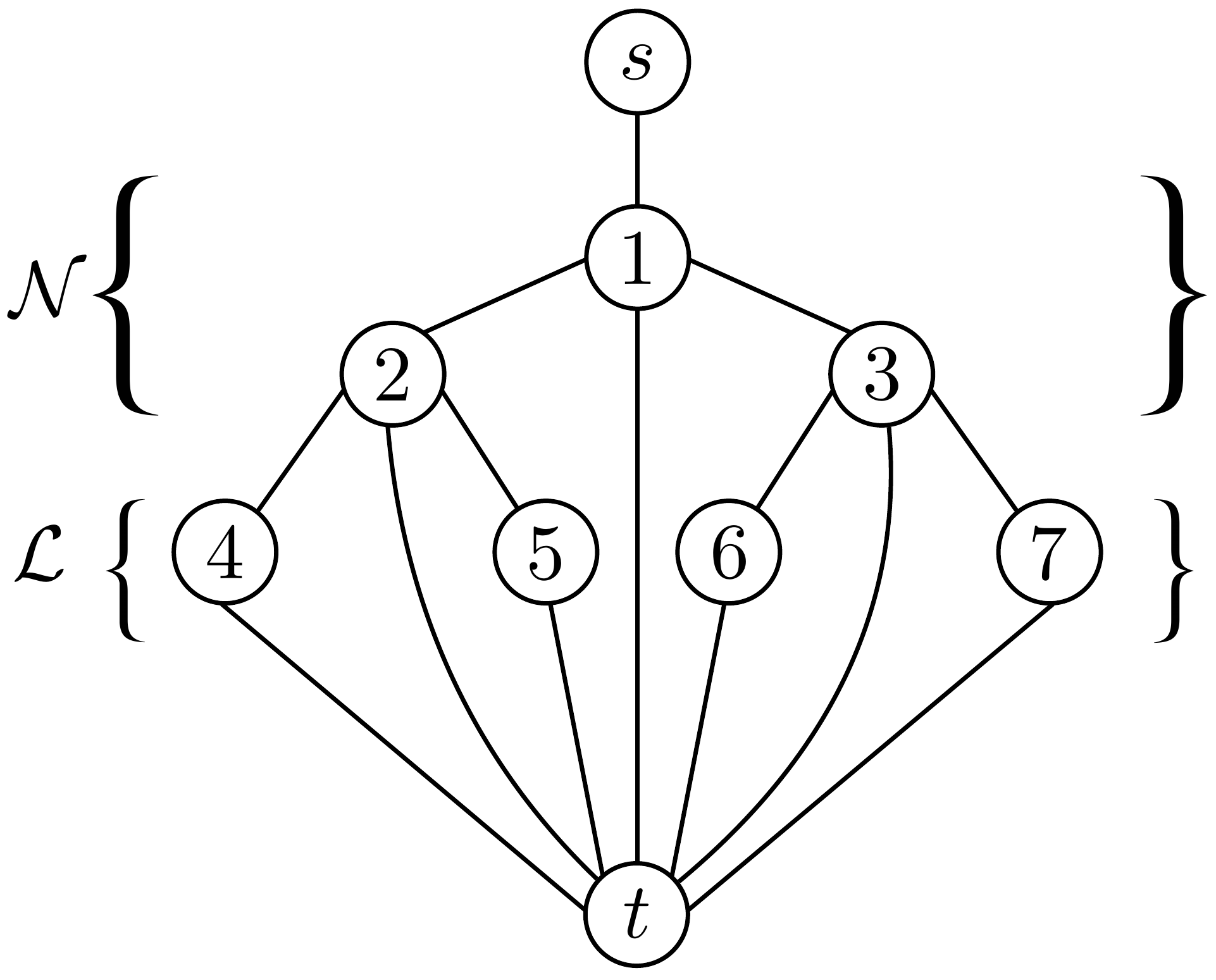}}
\caption{A classification tree of depth 2 viewed as a directed acyclic graph with a single source and sink.}
\label{fig:sample_tree}
\end{center}
\vskip -0.2in
\end{figure}

The flow-based formulation for decision trees reads
\begin{subequations}\label{eq:flow}
\begin{align}
\max \;\; & \displaystyle (1-\lambda)\sum_{i \in \mathcal I} \sum_{n \in \sets N \cup \sets L} z^i_{n,t}-\lambda\sum_{n \in \sets N}\sum_{f \in \sets F}b_{nf} \label{eq:mip_a}\\
\text{s.t.} \;\; & \displaystyle \sum_{f \in \sets F}b_{nf} + \sum_{k \in \sets K}w_{nk}=1   &\hspace{-5cm}  \forall n \in \sets N \label{eq:mip_h}\\
&  \displaystyle \sum_{k \in \sets K}w_{nk}=1 &\hspace{-5cm} \forall n \in \sets L \label{eq:mip_i}\\
& \displaystyle z^i_{a(n),n} =  z^i_{n,\ell(n)} + z^i_{n,r(n)} + z^i_{n,t}  &\hspace{-5cm}  \forall n \in \sets N, i \in \mathcal I \label{eq:mip_b}\\
&  \displaystyle z^i_{a(n),n} = z^i_{n,t} &\hspace{-5cm}   \forall i \in \mathcal I,n \in \sets L \label{eq:mip_c}\\
& \displaystyle z^i_{s,1} \leq 1 &\hspace{-5cm} \forall i \in \mathcal I\label{eq:mip_d}\\
&  \displaystyle z^i_{n,\ell(n)}\leq \sum_{f \in \sets F: x_{f}^i=0}b_{nf} &\hspace{-5cm} \forall n \in \sets N, i \in \mathcal I \label{eq:mip_e}\\
&  \displaystyle z^i_{n,r(n)}\leq \sum_{f \in \sets F: x_{f}^i=1}b_{nf}  &\hspace{-5cm} \forall n \in \sets N, i \in \mathcal I \label{eq:mip_f}\\
&  \displaystyle z^i_{n,t} \leq  w_{nk} &\hspace{-5cm} \forall i \in \mathcal I:\ y^i=k,n \in \sets N\cup \sets L\label{eq:mip_g}\\
&  \displaystyle b_{nf} \in \{0,1\}  &\hspace{-5cm}   \forall n \in \sets N,f \in \sets F \label{eq:mip_j}\\
&  \displaystyle w_{nk} \in \{0,1\}  &\hspace{-5cm}  \forall n \in \sets N \cup \sets L,k \in \sets K \label{eq:mip_k}\\
&  \displaystyle z^i_{a(n),n} \in \{0,1\}  &\hspace{-5cm}  \forall n \in \sets N \cup \sets L,i \in \sets I \label{eq:mip_l}\\
&  \displaystyle z^i_{n,t} \in \{0,1\}  &\hspace{-5cm}  \forall n \in \sets L,i \in \sets I \label{eq:mip_m},
\end{align}
\label{eq:mip}
\end{subequations}
where $\lambda \in [0,1]$ is a regularization weight. The objective \eqref{eq:mip_a} maximizes the total number of correctly classified points $\sum_{i \in \mathcal I} \sum_{n \in \sets N \cup \sets L} z^i_{n,t}$ while minimizing the number of splits $\sum_{n \in \sets N}\sum_{f \in \sets F}b_{nf}$. Thus, $\lambda$ controls the trade-off between these competing objectives, with larger values of lambda corresponding to greater regularization. An interpretation of the constraints is as follows. Constraint~\eqref{eq:mip_h} ensures that at each node we either branch on a feature or assign a class label to it (but not both, the label is only used if we do not branch at that node). Constraint~\eqref{eq:mip_i} guarantees that each leaf has a unique predicted class label. Constraint~\eqref{eq:mip_b} is a flow conservation constraint for each datapoint $i$ and node $n \in \sets N$: it ensures that if a datapoint arrives at a node, it must also leave the node through one of its descendants, or be correctly classified and routed to~$t$. Similarly, constraint~\eqref{eq:mip_c} enforces flow conservation for each node $n \in \sets L$. The inequality constraint~\eqref{eq:mip_d} ensures that at most one unit of flow can enter the graph through the source.  Constraints~\eqref{eq:mip_e} and~\eqref{eq:mip_f} ensure that if a datapoint is routed to the left (right) at node~$n$, then one of the features such that $x^i_f=0$ ($x^i_f=1$) must have been selected for branching at the node. Constraint~\eqref{eq:mip_g} ensures that datapoints routed to the sink node~$t$ are correctly classified.

Given a choice of branching and labeling decisions, $b$ and $w$, each datapoint is allotted one unit of flow which it attempts to guide through the graph from the source node to the sink node. If the datapoint cannot be correctly classified, the flow that will reach the sink (and by extension enter the source) will be zero. In particular note that once the $b$ and $w$ variables have been fixed, optimization of the flows can be done separately for each datapoint. This implies that the problem can be decomposed to speed-up computation, an idea that we leverage in Section~\ref{sec:Bender}. In particular, note that the optimization over flow variables can be cast as a max-flow problem for each datapoint, implying that the integrality constraint on the~$z$ variables can be relaxed to yield an equivalent formulation. We leverage this idea in our computational experiments.

Formulation~\eqref{eq:mip} has several distinguishing features relative to existing MIP formulations for training decision trees
\begin{enumerate}[label=\emph{\roman*)}]\setlength\itemsep{0em}
   \item It does not use big-$M$ constraints.
   \item 
   It includes \emph{flow variables} indicating whether each datapoint is directed to the left or right at each branching node.
   \item It only tracks datapoints that are correctly classified.
   \end{enumerate}


The number of variables and constraints in formulation \eqref{eq:flow} is $\sets O\big(2^{d}(|\sets I|+|\sets F|)\big)$, where $d$ is the tree depth. Thus, its size is of the same order as the one proposed by \citet{bertsimas2017optimal}. Nonetheless, as we discuss in \S\ref{sec:lprelax}, the LP relaxation of formulation \eqref{eq:flow} is tighter, and therefore results in a more aggressive pruning of the search space without incurring in significant additional costs.

\subsection{Strength of the Flow-Based Formulation}\label{sec:lprelax}

We now argue that formulation \eqref{eq:flow}, which we henceforth refer to as \emph{flow-based formulation}, is stronger than existing formulations from the literature. 
The BinOCT formulation of \citet{verwer2019learning} is obtained by aggregating constraints from the OCT formulation of \citet{bertsimas2017optimal} (using big-$M$ constants). As a consequence, its relaxation is weaker. Thus, it suffices to argue that the proposed formulation is stronger than OCT. 
\begin{proposition}\label{prop:OCT}
If $\lambda=0$, then formulation \eqref{eq:flow} has a stronger relaxation than OCT.
\end{proposition}
A formal proof of Proposition~\ref{prop:OCT} is given in online companion C. In the following, we provide some intuition in how formulation \eqref{eq:flow} is stronger.
We work with a simplified version of the formulation of \citet{bertsimas2017optimal} specialized to the case of binary data. We provide this formulation in the online companion B.


\subsubsection{No big-$M$s}

In this section, we argue that the absence of big-$M$ constraints in our formulation induces a stronger formulation. In the OCT formulation, for $i\in \sets I$ and $n\in \sets L$, there are binary variables $\zeta$ such that $\zeta_{a(n),n}^i=1$ if datapoint $i$ is assigned to leaf node $n$ (regardless of whether that point is correctly classified or not), and $\zeta_{a(n),n}^i=0$ otherwise. In addition, the authors introduce a variable $L_n$ that represents the number of missclassified points at leaf node $n$, and this variable is defined via constraints $L_n\geq 0$ and 

\begin{equation*}\label{eq:bigM}
L_n\geq \sum_{i\in \sets I}\zeta_{a(n),n}^i-\sum_{ \begin{smallmatrix} i\in \sets I:\\ y^i=k \end{smallmatrix}}\zeta_{a(n),n}^i-|\sets I|(1-w_{nk})\;\; \forall k\in \sets K.
\end{equation*}
Thus, the number of correctly classified points is $|\sets I|-\sum_{n\in \sets L}L_n$. Note that this is a big-$M$ constraint, with $M=|\sets I|$, which is activated or deactivated depending on whether $w_{nk}=1$ or not.

The LP relaxation induced from counting correctly classified points can be improved. The number of such points, using the variables above, is 
\begin{equation}\label{eq:quadratic}|\sets I|-\sum_{n\in \sets L}L_n=\sum_{n\in \sets L}\sum_{i\in \sets I}\zeta_{a(n),n}^i w_{ny^i}.
\end{equation} The right hand side of \eqref{eq:quadratic} is nonlinear (quadratic). Nonetheless, the quadratic function is \emph{supermodular}, see \citet{nemhauser1978analysis}, and its concave envelop can be described by introducing variables $z_{a(n),n}^i:=\zeta_{a(n),n}^i w_{ny^i}$ via the system
\begin{align*}
    &|\sets I|-\sum_{n\in \sets L}L_n\leq \sum_{n\in \sets L}\sum_{i\in \sets I}z_{a(n),n}^i\\
    &z_{a(n),n}^i\leq \zeta_{a(n),n}^i,\;  z_{a(n),n}^i\leq w_{n y^i} \quad \forall n \in \sets N, \; i\in \sets I.
\end{align*}
The additional variables $z$ are precisely the variables used in formulation~\eqref{eq:flow}. Note that a simple application of this idea would require the introduction of additional variables for each pair $(i,n)$. However, by noting that the desired tree structure can be enforced  using the new variables $z$ only, and the original variables $\zeta$ can be dropped, we achieve this strengthening without incurring the cost of a larger formulation.

\subsubsection{Improved branching constraints}\label{sec:branching}
To correctly enforce the branching structure of the decision-tree, \citet{bertsimas2017optimal} use (after specializing their formulation to the case of binary data) constraints of the form 
\begin{equation}\label{eq:branching}
\begin{array}{l}
    z_{a(m),m}^i\leq 1-b_{nf}\;\; \forall i\in \sets I,  m\in \sets L, n\in\sets {AL}(m), f \in \sets F: x_f^i=1,
\end{array}
\end{equation}
where $\sets {AL}(m)$ denotes the set of ancestors of $m$ whose left branch was followed on the path from the root to $m$. An intrepretation of this constraint is as follows: if datapoint $i$ reaches leaf node $m$, then for all nodes in the path where~$i$ took the left direction, no branching decision~$b_{nf}$ can be made that would cause the point to go right. Instead, we use constraint~\eqref{eq:mip_e}.

We now show that \eqref{eq:mip_e} induces a stronger LP relaxation. First, we focus on the left hand side of \eqref{eq:mip_e}: due to flow conservation constraints \eqref{eq:mip_b}, we find that $$z_{n,\ell(n)}^i=\sum_{m\in \sets L: m\in \sets{LD}(n)}z_{a(m),m}^i$$ where, following the notation of~\citet{bertsimas2017optimal}, $\sets{LD}(n)$ is the set of left descendants of $n$. In particular, the left hand side of constraint \eqref{eq:mip_e} is larger than the left hand side of \eqref{eq:branching}. Now, we focus on the right hand side: from constraints \eqref{eq:mip_h}, we find that 
$$\sum_{f \in \sets F: x_{f}^i=0}b_{nf}=1-\sum_{k\in \sets K}y_{nk}-\sum_{f\in \sets F: x_f^i=1}b_{nf}.$$
In particular, the right hand side of \eqref{eq:mip_e} is smaller than the right hand side of \eqref{eq:branching}. Similar arguments can be made for constraint~\eqref{eq:mip_f}. As a consequence, the linear inequalities for branching induced from formulation \eqref{eq:flow} dominate those proposed by \citet{bertsimas2017optimal}.

\subsubsection{Further Strengthening of the Formulation}
\label{sec:strenghtening}

Formulation \eqref{eq:flow} can be strengthened even more through the addition of cuts.

Let $n\in \sets N$ be any node such that $\ell(n)$ and $r(n)\in \sets L$. Also, let $f\in \sets F$ and define $\sets H\subseteq \sets I$ as any subset of the rows such that: \textit{a)} $i\in \sets H \Rightarrow x^i_f=1$, and \textit{b)} $i,j\in \sets H \Rightarrow y^i\neq y^j$. Intuitively, $\sets H$ is a set of points belonging to different classes that would all be assigned to the right branch if feature $f$ is selected for branching. Then, the constraint 
\begin{equation}\label{eq:multi}\sum_{i\in \sets H}z^i_{n,\ell(n)}\leq 1-b_{n,f}
\end{equation} 
is valid: indeed, if $b_{n,f}=1$, then none of the points in $\sets H$ can be assigned to the left branch; and, if $b_{n,f}=0$, then at most one of the points in $\sets H$ can be correctly classified.

None of the constraint in \eqref{eq:flow} implies \eqref{eq:multi}. As a matter of fact, if all constraints \eqref{eq:multi} are added for all possible combinations of sets $\sets H$, nodes $n$ and features $f$, then variables $w_{nk}$ with $n\in \sets L$ could be dropped from the formulation, along with constraints \eqref{eq:mip_g} and \eqref{eq:mip_i}. Naturally, we do not add all constraints \eqref{eq:multi} a priori, but instead use cuts to enforce them as needed.

\section{\rev{A Benders' \newpv{D}ecomposition \newpv{A}pproach}}
\label{sec:Bender}
The flow-based formulation~\eqref{eq:mip} is effective at reducing the number of branch-and-bound nodes required to prove optimality when compared with existing formulations, and results in a substantial speedup in small- and \newpv{medium-sized} instances, see \S\ref{sec:Experiments}. However, in larger instances, the computational time required to solve the LP relaxations may become prohibitive, impairing its performance in branch-and-bound. 

Recall from \S\ref{sec:DT_Formulation} that, if variables $b$ and $w$ are fixed, then the problem decomposes into $|\mathcal{I}|$ independent \emph{subproblems}, one for each datapoint. Additionally, each problem is a maximum flow problem, for which specialized polynomial-time methods exist. Due to these characteristics, formulation~\eqref{eq:mip} can be naturally tackled using Benders' decomposition, see~\citet{benders1962partitioning}. In what follows, we \rev{describe the Benders' decomposition approach}.

Observe that problem \eqref{eq:flow} can be written in an equivalent fashion by making the subproblems explicit as follows:
\begin{subequations}
\begin{align}
\max \;\;&   \displaystyle (1-\lambda)\sum_{i \in \sets I}g^i(b,w) -\lambda \sum_{n\in \sets N}\sum_{f\in \sets F}b_{nf} \label{eq:mip_MP_a}\\
\text{s.t.} \; \; & \displaystyle \sum_{f \in \sets F}b_{nf} + \sum_{k \in \sets K}w_{nk}=1 &\hspace{-5cm} \forall n \in \sets N \label{eq:mip_MP_b}\\
&  \displaystyle \sum_{k \in \sets K}w_{nk}=1 &\hspace{-5cm}  \forall n \in \sets L \label{eq:mip_MP_c}\\
&  \displaystyle b_{nf} \in \{0,1\}  &\hspace{-5cm} \forall n \in \sets N,f \in \sets F \label{eq:mip_MP_d}\\
&  \displaystyle w_{nk} \in \{0,1\}  &\hspace{-5cm} \forall n \in \sets N \cup \sets L, k \in \sets K \label{eq:mip_MP_e},
\end{align}
\label{eq:mip_MP}
\end{subequations}
where, for any fixed $i\in \sets I$, $w$ and $b$, $g^i(b,w)$ is defined as the optimal objective value of the max-flow problem
\begin{subequations}
\begin{align}
\max \;\;&  \displaystyle \sum_{n \in \sets N \cup \sets L}z^i_{n,t}\label{eq:mip_SP_a}\\
\text{s.t.} \; \; & \displaystyle z^i_{a(n),n} =  z^i_{n,\ell(n)} + z^i_{n,r(n)}+ z^i_{n,t} &\hspace{-5cm} \forall n \in \sets N \label{eq:mip_SP_b}\\
&  \displaystyle z^i_{a(n),n}=z^i_{n,t} &\hspace{-5cm} \forall n \in \sets L \label{eq:mip_SP_c}\\
&  \displaystyle z^i_{s,1}\leq c^i_{s,1}(b,w)  \label{eq:mip_SP_d}\\
&  \displaystyle z^i_{n,\ell(n)} \leq c^i_{n,\ell(n)}(b,w) &\hspace{-5cm}  \forall n \in \sets N \label{eq:mip_SP_e}\\
&  \displaystyle z^i_{n,r(n)} \leq c^i_{n,r(n)}(b,w) &\hspace{-5cm} \forall n \in \sets N \label{eq:mip_SP_f}\\
&  \displaystyle z^i_{n,t} \leq  c^i_{n,t}(b,w) &\hspace{-5cm} \forall n \in \sets N \cup \sets L \label{eq:mip_SP_g}\\
&  \displaystyle z^i_{a(n),n} \geq 0 &\hspace{-5cm} \forall n \in \sets N \cup \sets L \label{eq:mip_SP_h}\\
&  \displaystyle z^i_{n,t} \geq 0 &\hspace{-5cm} \forall n \in \sets N \cup \sets L \label{eq:mip_SP_i}.
\end{align}
\label{eq:mip_SP}
\end{subequations}
In formulation \eqref{eq:mip_SP} we use the shorthand $c_{nn'}(b,w)$ to represent upper bounds on the decision variables $z$. These values can be interpreted as edge capacities in the flow problem, and are given as $c^i_{s,1}(b,w):=1$ for all $n \in \sets N $, 
 $c^i_{n,\ell(n)}(b,w):= \sum_{f \in \sets F: x_{f}^i=0}b_{nf}$ and
$ c^i_{n,r(n)}(b,w):= \sum_{f \in \sets F: x_{f}^i=1}b_{nf}$ for all $n \in \sets N \cup \sets L$, and finally $c^i_{n,t}(b,w):=w_{ny^i}$.
Note that~$g^i(b,w)=1$ if point $i$ is correctly classified given the tree structure and class labels induced by $(b,w)$. 

From the well-known max-flow/min-cut duality, we find that $g^i(b,w)$ also equals the optimal value of the dual of the above max-flow problem, which is expressible as
\begin{subequations}
\begin{align}
\min \;\;& \displaystyle c^i_{s,1}(b,w)q_{s,1}+\sum_{n \in \sets N}c^i_{n,\ell(n)}(b,w)q_{n,\ell(n)} + \sum_{n \in \sets N}c^i_{n,r(n)}(b,w)q_{n,r(n)}  +\sum_{n \in \sets N \cup \sets L}c^i_{n,t}(b,w)q_{n,t}  \label{eq:mip_DSP_a}\\
\text{s.t.}\; \; & \displaystyle q_{s,1} + p_1  \geq 1\label{eq:mip_DSP_b}\\
&  \displaystyle q_{n,\ell(n)}+p_{\ell(n)} - p_n \geq 0 &\hspace{-5cm} \forall n \in \sets N\quad \label{eq:mip_DSP_c}\\
&  \displaystyle q_{n,r(n)}+p_{r(n)} - p_n \geq 0 &\hspace{-5cm} \forall n \in \sets N\quad\label{eq:mip_DSP_d}\\
&  \displaystyle q_{n,t}- p_n \geq 0 &\hspace{-5cm} \forall n \in \sets N \cup \sets L\quad\label{eq:mip_DSP_e}\\
&  \displaystyle q_{s,1} \geq 0 &\hspace{-5cm} \label{eq:mip_DSP_f}\quad\\
&  \displaystyle q_{n,\ell(n)},q_{n,r(n)} \geq 0 &\hspace{-5cm} \forall n \in \sets N \quad\label{eq:mip_DSP_g}\\
&  \displaystyle q_{n,t} \geq 0 &\hspace{-5cm} \forall n \in \sets N \cup \sets L \quad\label{eq:mip_DSP_h}.
\end{align}
\label{eq:mip_DSP}
\end{subequations}
Problem~\eqref{eq:mip_DSP} is a minimum cut problem, where variable $p_n$ is one if and only if node $n$ is in the source set (we implicitly fix $p_s=1$), and variable $q_{i,j}$ is one if and only if arc $(i,j)$ is part of the minimum cut. Note that the feasible region \eqref{eq:mip_DSP_b}-\eqref{eq:mip_DSP_h} of the minimum cut problem does not depend on the variables $(b,w)$; we denote this feasible region by $\sets P$.

We can now reformulate the master problem~\eqref{eq:mip_MP} as follows:
\begin{subequations}
\begin{align}
 \max \;\;& \displaystyle (1-\lambda)\sum_{i \in \sets I}g^i -\lambda\sum_{n \in \sets N}\sum_{f \in \sets F}b_{nf} \label{eq:mip_MP2_a}\\
 \text{s.t.}\;\;&  
 g^i \leq  \displaystyle c^i_{s,1}(b,w)q_{s,1}+\sum_{n \in \sets N}c^i_{n,\ell(n)}(b,w)q_{n,\ell(n)}
+ \sum_{n \in \sets N}c^i_{n,r(n)}(b,w)q_{n,r(n)} \nonumber \\
&  \qquad \quad \qquad +\sum_{n \in \sets N \cup \sets L}c^i_{n,t}(b,w)q_{n,t} &\hspace{-5cm}  \newpv{ \forall q: (p,q)\in \sets P }  \label{eq:mip_MP2_b}\\
&  \displaystyle \sum_{f \in \sets F}b_{nf}+ \sum_{k \in \sets K}w_{nk}=1   &\hspace{-5cm} \forall n \in \sets N  \label{eq:mip_MP2_c}\\
&  \displaystyle \sum_{k \in \sets K}w_{nk}=1  &\hspace{-5cm}  \forall n \in \sets L \label{eq:mip_MP2_d}\\
&  \displaystyle g^i \leq 1   & \forall i \in \sets I \label{eq:mip_MP2_e}\\
&  \displaystyle b_{nf} \in \{0,1\}  &\hspace{-5cm} \forall n \in \sets N,f \in \sets F \label{eq:mip_MP2_f}\\
&  \displaystyle w_{nk} \in \{0,1\}  &\hspace{-5cm} \forall n \in \sets N\cup \sets L,k \in \sets K . \label{eq:mip_MP2_g}
\end{align}
\label{eq:mip_MP2}
\end{subequations}

In the above formulation, we have added constraint~\eqref{eq:mip_MP2_e} to make sure we get bounded solutions in the relaxed master problem. \rev{\newpv{Note that constraint~\eqref{eq:mip_MP2_b} can be relaxed to only hold $\forall q: (p,q)\in \text{ext}(\sets P)$, where $\text{ext}(\sets P)$ denotes the extreme points of $\sets P$}. These extreme points correspond to cuts induced by \eqref{eq:mip_DSP_b}-\eqref{eq:mip_DSP_h} \newpv{in the graph}. Moreover, \newpv{observe} that equalities \eqref{eq:mip_MP2_c} \newpv{and} \eqref{eq:mip_MP2_d} can be relaxed to \newpv{inequalities} without loss of generality. \newpv{Indeed, in} any feasible solution where $\sum_{f\in\sets F}b_{nf}+\sum_{k\in \sets K}w_{nk}<1$ for some $n\in \sets N$, it is possible to set any \newpv{$w_{nk}$ to unity} to obtain a feasible solution with identical objective value \newpv{and} where \eqref{eq:mip_MP2_c} is satisfied at equality. We define $\sets H_=$ as the set of $(b,w,g)$ satisfying constraints \eqref{eq:mip_MP2_b}-\eqref{eq:mip_MP2_g}, and define $\sets H_\leq$ as the set of points satisfying the inequality version of  \eqref{eq:mip_MP2_b}-\eqref{eq:mip_MP2_g}. In the next section we discuss effective implementations of problem~\eqref{eq:mip_MP2}.}

\section{\newpv{Generating Strong Cuts on the Fly}}

Formulation~\eqref{eq:mip_MP2} \rev{contains an exponential number of inequalities \eqref{eq:mip_MP2_b}, and needs to be} implemented using row generation, where\newpv{in} constraint\newpv{s}~\eqref{eq:mip_MP2_b} 
\newpv{are} initially dropped \newpv{and added as} cuts on the fly \newpv{during optimization}. 
Row generation can be implemented in modern MIP optimization solvers via callbacks, by adding lazy constraints at relevant nodes of the branch-and-bound tree. Identifying which constraint \eqref{eq:mip_MP2_b} to add can in general be done by solving a minimum cut problem, \rev{and could in principle be solved via well-known} algorithms, \rev{such as}~\citet{goldberg1988new} and \citet{hochbaum2008pseudoflow}.

\rev{\newpv{Row generation} methods for integer programs may require a long time to converge to an optimal solution if each cut added is weak for the feasible region of interest, as illustrated for example by the poor performance of the pure cutting plane algorithm of \citet{gomory1958outline}. Nonetheless, cutting planes have been \newpv{extremely} successful \newpv{at} solving integer programs when the cuts added are strong or, ideally, ``facet-defining'' for the convex hull of the feasible region. \newpv{Formally, facet-defining cuts are those cuts which are necessary to describe the convex hull}. For example, integer programming formulations for traveling salesman problems contain an exponential number of ``subtour elimination'' constraints that are added on the fly as cuts. Nonetheless, all such inequalities are facet defining for the convex hull of the feasible region, see \citet{grotschel1985polyhedral}, and \newpv{cutting plane methods} are able to find \newpv{provably} optimal tours to problems with tens of thousands of variables or more \cite{applegate2009certification}.}
\rev{Unfortunately, \newpv{as illustrated by the following example}, several of \newpv{the} inequalities \eqref{eq:mip_MP2_b} may actually be weak for $\text{conv}(\sets H_=)$ and $\text{conv}(\sets H_\leq)$, where $\text{conv}(\sets H)$ denotes the convex hull of $\sets H$. 
\begin{example}\label{ex:counterexample}
		Consider a\newpv{n instance of Problem~\eqref{eq:mip_MP2} with} a depth \newpv{$d=1$ decision-tree} (\newpv{i.e., $\sets N=\{1\}$ and $\sets{L}=\{2,3\}$}) \newpv{and a dataset involving} a single feature ($\sets F =\{1\}$). Consider datapoint $i$ such that \newpv{$x_1^i=0$ and $y^i=k$}. \newpv{Suppose that the solution to the master problem is such that we branch on (the unique) feature at node 1 and predict class $k'\neq k$ at node 2. Then, datapoint $i$ is routed left at node 1 and is misclassified. A valid min-cut for the resulting graph includes all arcs incoming into the sink, } i.e., \newpv{$q_{n,n'}=1$ iff $n'= t$. The associated cut~\eqref{eq:mip_MP2_b} reads} 
		\begin{equation}\label{eq:weak}g^i\leq w_{1k}+w_{2k}+w_{3k}.\end{equation}
		Intuitively, \eqref{eq:weak} states that datapoint $i$ can be correctly classified if its class label is assigned to at least one node, and is valid for $\text{conv}(\sets H_=)$ and $\text{conv}(\sets H_\leq)$. However, since datapoint $i$ cannot be routed to node $3$, the stronger inequality
		\begin{equation}\label{eq:strong}g^i\leq w_{1k}+ w_{2k}\end{equation} is valid for $\text{conv}(\sets H_=)$, $\text{conv}(\sets H_\leq)$ and dominates \eqref{eq:weak}.~\qed
\end{example}

Therefore, an implementation of formulation \eqref{eq:mip_MP2} using general purpose min-cut algorithms to identify constraints to add may perform poorly.  \newpv{This motivates us to develop} a tailored algorithm that exploits the structure of the graph induced by capacities~$c(b,w)$. \newpv{As we will show, our algorithm exhibits substantially faster runtimes than general purpose min-cut methods and returns inequalities that are never dominated, resulting in faster convergence of the Benders' decomposition approach.}} 

\rev{Algorithm~\ref{alg:cut} shows the proposed \newpv{procedure}, which can be called at \emph{integer nodes} of the branch-and-bound tree. \newpv{For notational convenience}, we define $b_{n,\ell(n)}=b_{n,r(n)}=0$ for leaf nodes $n\in \sets L$. Since at each iteration in the main loop (lines~\ref{line:ini}-\ref{line:end}), the value of $n$ is updated to a descendant of $n$, the algorithm terminates in a most $\sets O(d)$ iterations, where~$d$ is the depth of the tree -- since $|\sets N\cup \sets L|$ is $\sets O(2^d)$, the complexity is logarithmic in the size of the tree.} Figure~\ref{fig:subproblem_tree} illustrates graphically Algorithm~\ref{alg:cut}. \rev{We now prove that Algorithm~\ref{alg:cut} is indeed a valid \emph{separation algorithm}.}

\newpv{
\begin{algorithm}[h]
		\caption{Separation \newpv{procedure}}
		\label{alg:cut}
		 \textbf{Input:} $(b,w,g) \text{ satisfying~\eqref{eq:mip_MP2_c}-\eqref{eq:mip_MP2_g}; } \newline
		 \hspace*{\algorithmicindent} i~\in~\sets I:  \text{ datapoint used to generate the cut.}$ \newline
		 \textbf{Output:} $-1$ if all constraints \eqref{eq:mip_MP2_b} are satisfied;  \newline 
		 \hspace*{\algorithmicindent} values \newpv{for min-cut} $q$ otherwise.
	    \begin{algorithmic}[1]
 		\State \textbf{if }$g^i=0$ \textbf{ return } \newpv{$-1$}\label{line:simple}
 		\State \textbf{Initialize} $q\leftarrow 0$ \hfill \Comment{No arcs in the cut}
 		\State \textbf{Initialize} $n\leftarrow 1$ \hfill \Comment{Current node=root}
 		\State \textbf{Initialize} $\sets S\leftarrow \{s\}$ \hfill \Comment{$\sets S$ is the source set of the cut}
 		\Loop\label{line:ini}
 		\State $\sets S\leftarrow \sets S\cup\{n\}$
 		\If{$c_{n,\ell(n)}^i(b,w)=1$}  \label{lin:sub-start}
 		\State $q_{n,r(n)}\leftarrow 1$ \hfill
 		\Comment{Arcs to the right are in the cut}
 		\State $q_{n,t}\leftarrow 1$ \hfill \Comment{Arcs to the sink are in the cut}
 		\State $n\leftarrow \ell(n)$ \hfill \Comment{Datapoint $i$ is routed left}
 		\ElsIf{$c_{n,r(n)}^i(b,w)=1$} 
 		\State $q_{n,\ell(n)}\leftarrow 1$ \hfill
 		\Comment{Arcs to the left are in the cut}
 		\State $q_{n,t}\leftarrow 1$ \hfill \Comment{Arcs to the sink are in the cut}
 		\State $n\leftarrow r(n)$ \hfill \Comment{Datapoint $i$ is routed right} \label{lin:sub-end}
 		\ElsIf{$c_{n,t}^i(b,w)=0$}
 		\State $q_{n,\ell(n)}\leftarrow 1$ \hfill
 		\Comment{Arcs to the left are in the cut}
 		\State $q_{n,r(n)}\leftarrow 1$ \hfill \Comment{Arcs to the right are in the cut}
 		\State \newpv{$q_{n,t}\leftarrow1$} \hfill \Comment{Datapoint $i$ is misclassified}
 		\State \textbf{return }$q$ \label{line:violated}
 		\Else \newpv{\Comment{$c_{n,t}^i(b,w)=1$ in this case}}
 		\State \textbf{return }$-1$ \newpv{\hfill \Comment{$i$ is correctly classified}} \label{line:satisfied}
 		\EndIf
 		\EndLoop \label{line:end}
	\end{algorithmic}
\end{algorithm}
}
\begin{figure}[tb]
\vskip 0.2in
\begin{center}
\centerline{\includegraphics[width=0.5\textwidth]{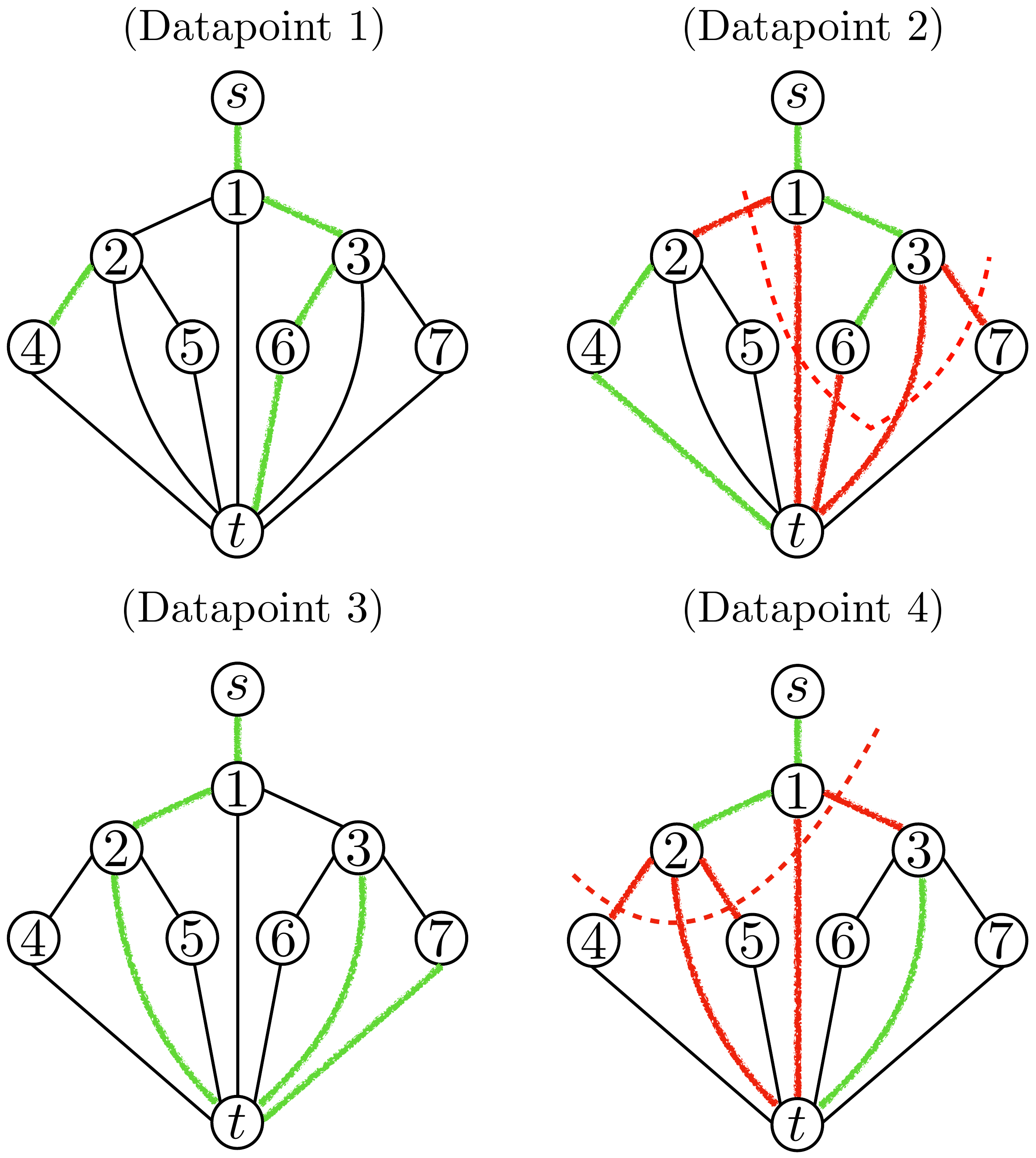}}
\caption{Pictorial description of \rev{Algorithm~\ref{alg:cut}}. Green arcs \newpv{$(n,n')$} have capacity \rev{$c_{n,n'}^i(b,w)=1$} (and others capacity 0). Red arcs are those in the minimum cut. On the left, examples with minimum cut value equal to 1 (constraint~\eqref{eq:mip_MP2_b} is satisfied). On the right, examples with minimum cut values of 0 (new constraints added).}
\label{fig:subproblem_tree}
\end{center}
\vskip -0.2in
\end{figure}

\begin{proposition}\label{prop:algorithm}
 \rev{Given $i\in \sets I$ and $(b,w,g)$ satisfying~\eqref{eq:mip_MP2_c}-\eqref{eq:mip_MP2_g}, Algorithm~\ref{alg:cut} either finds a violated inequality \eqref{eq:mip_MP2_b} or proves that all such inequalities are satisfied}.
\end{proposition}
\begin{proof}
\rev{Note that the right hand side of \eqref{eq:mip_MP2_b}, which corresponds to the capacity of a cut \newpv{in the graph}, is nonnegative. Therefore, if $g^i=0$ (line~\ref{line:simple}), all inequalities are automatically satisfied.} Since $(b,w)$ is integer, all \newpv{arc} capacities in formulations~\eqref{eq:mip_SP} and~\eqref{eq:mip_DSP}  are either~0 or~1. \rev{Moreover, since $g^i\leq 1$, we find that either the value of a minimum cut is~$0$ and there exists a violated inequality, or the value of a minimum cut is at least~$1$ and there is no violated inequality. Finally, there exists a 0-capacity cut if and only if $s$ and~$t$ belong to different connected components in the graph induced by~\newpv{$c(b,w)$}. 

\newpv{The component connected to $s$ can be found using depth-first search. For any fixed $n \in \sets N \cup \sets L$, constraints~\eqref{eq:mip_MP2_c}-\eqref{eq:mip_MP2_d} and the definition of $c(b,w)$ imply that at most one arc $(n,n')$ outgoing from~$n$ can have capacity 1. If arc $(n,n')$ has capacity 1 and $n' \neq t$ (lines~\ref{lin:sub-start}-\ref{lin:sub-end}), then $n'$ can be added to the component connected to~$s$ (set $\sets S$) and all other outgoing arcs from $n$ (which have capacity of 0) can be added to the min-cut (at zero cost). If all outgoing arcs from \newpv{$n$} have capacity $0$, they can be added to the min-cut. In that case, the connected components to $s$ end at node $n$. If the unique outgoing arc from node $n$ that has capacity 1 is $(n,t)$, then $s$ and $t$ are in the same connected component and the value of the minimum cut is at least 1. Therefore, the connected component~\newpv{$\sets S$} to $s$ corresponds to a path from $s$ to a node $n$ where no branching is performed: if $c_{n,t}^i=1$ then $t$ is also in this connected component and no cut is added (line~\ref{line:satisfied}): otherwise, a violated cut has been \newpv{found} (line \ref{line:violated}).}
}
\end{proof}

\rev{In addition \newpv{to} providing a very fast method \newpv{for generating} cuts at integer nodes of a branch-and-bound tree, Algorithm~\ref{alg:cut} is also guaranteed to generate facet-defining cuts of \newpv{${\textrm{conv}}(H_\leq)$}. \newpv{Such cuts are never dominated.}}

\begin{theorem}\label{theo:facet}
All violated inequalities found by Algorithm~\ref{alg:cut} are facet-defining for \rev{$\text{conv}(\sets H_\leq)$}.
\end{theorem}
We defer the proof of Theorem~\ref{theo:facet} to the supplemental material~\ref{appendix_sec:polyhedral}.

\rev{
	\begin{example}[\newpv{Example~\ref{ex:counterexample} Continued}]
		In the instance considered in Example~\ref{ex:counterexample}, if $b_{1f}=1$ and $w_{2k}=0$, then the cut generated by the algorithm (\newpv{$q_{1,r(1)} = q_{1,t}$=$q_{2,t}=1$}) is precisely \eqref{eq:strong}. If $b_{1f}=0$ and $w_{1k}=0$ \newpv{(which is feasible in ${\textrm{conv}}(H_\leq)$)} in the \newpv{solution to the master problem} used to generate the cut, then the cut \newpv{returned} by Algorithm~\ref{alg:cut} (\newpv{$q_{1,\ell(1)} = q_{1,t}=1$}) is
		$$g^i\leq w_{1k}+b_{1f}.$$
		 For all other possible values of $(b,w)$, Algorithm~\ref{alg:cut} does not find a violated cut.  
	\end{example}
}

\section{Experiments}
\label{sec:Experiments}
\paragraph{Approaches and Datasets.} We evaluate our two approaches on eight publicly available datasets. The number of rows ($\sets I$), number of one-hot encoded features ($\sets F$), and number of class labels ($\sets K$) for each dataset are given in Table~\ref{tab:datasets}.
We compare the flow-based formulation (\texttt{FlowOCT}) and its Benders' decomposition (\texttt{Benders}) 
against the formulations proposed by~\citet{bertsimas2017optimal} (\texttt{OCT}) and~\citet{verwer2019learning} (\texttt{BinOCT}). 
As the code used for OCT is not publicly available, we implemented the corresponding formulation (adapted for the case of binary data). The details of this implementation are given in the online companion B.
\begin{table}[h]
\caption{Datasets used in the experiments.}
\label{tab:datasets}
\vskip 0.15in
\begin{center}
\begin{tabular}{lccc}
\hline
Dataset                          & $|\sets I|$ & $|\sets F|$ & $|\sets K|$ \\ \hline
monk3                         & 122                  & 15                  & 2                \\
monk1                       & 124                  & 15                  & 2                \\ 
monk2                         & 169                  & 15                  & 2                \\ 
house-votes-84     & 232                  & 16                  & 2                \\
balance-scale                    & 625                  & 20                  & 3                \\ 
tic-tac-toe                     & 958                  & 27                  & 2                \\
car\_evaluation                  & 1728                 & 20                  & 4                \\
kr-vs-kp     & 3196                 & 38                  & 2                \\ \hline
\end{tabular}
\end{center}
\vskip -0.1in
\end{table}

\paragraph{Experimental Setup.} Each dataset is split into three parts: the training set (50\%), the validation set (25\%), and the testing set (25\%). The training and validation sets are used to tune the value of the hyperparameter $\lambda$. We repeat this process 5 times with 5 different samples. We test values of $\lambda=0.1j$ for $j=0,\ldots,9$. Finally, we use the best $\lambda$ to train a tree using the training and evaluation sets from the previous step, which we then evaluate against the testing set to determine the out-of-sample accuracy.
All approaches are implemented in Python programming language and solved by the Gurobi 8.1 solver.
All problems are solved on a single core of SL250s Xeon CPUs by HPE and 4gb of memory with a one hour time limit.

\begin{figure*}[!bt]
    \centering
    \subfloat[Performance profile with $\lambda=0$.]{
    \includegraphics[width=0.5\columnwidth]{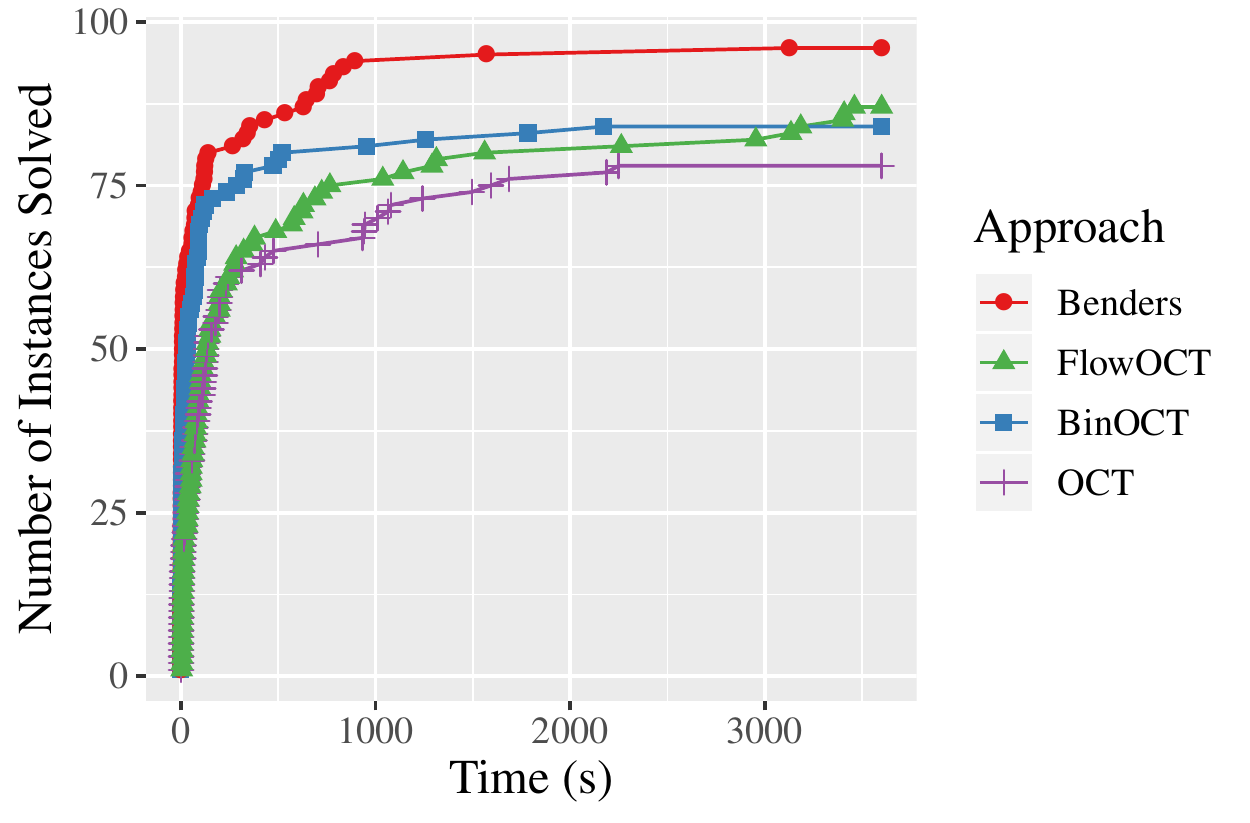}} \hspace{0.25cm}
    \subfloat[Performance profile with $\lambda>0$.]{
    \includegraphics[width=0.5\columnwidth]{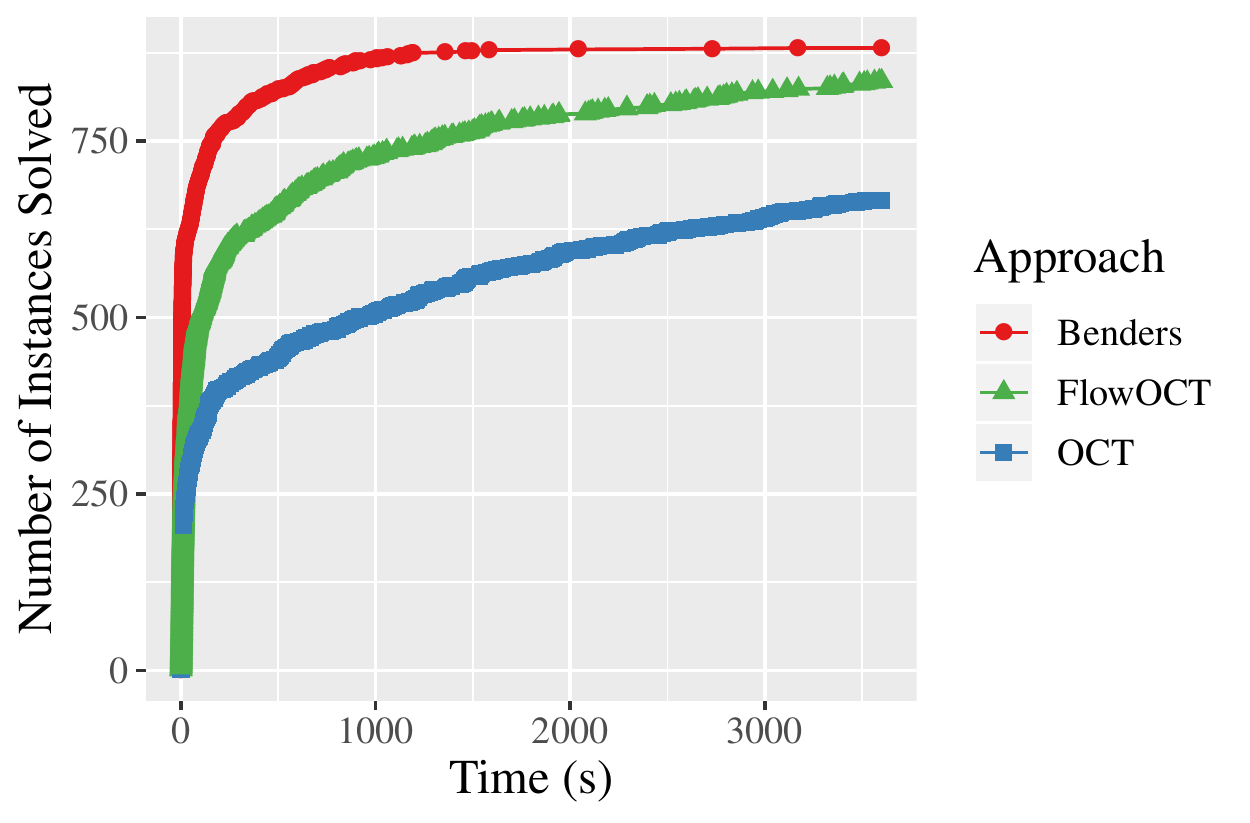}} \hspace{0.25cm}
	\subfloat[Optimality gaps as a function of the size$=2^d\times |\sets I|$.]{
    \includegraphics[width=0.5\columnwidth]{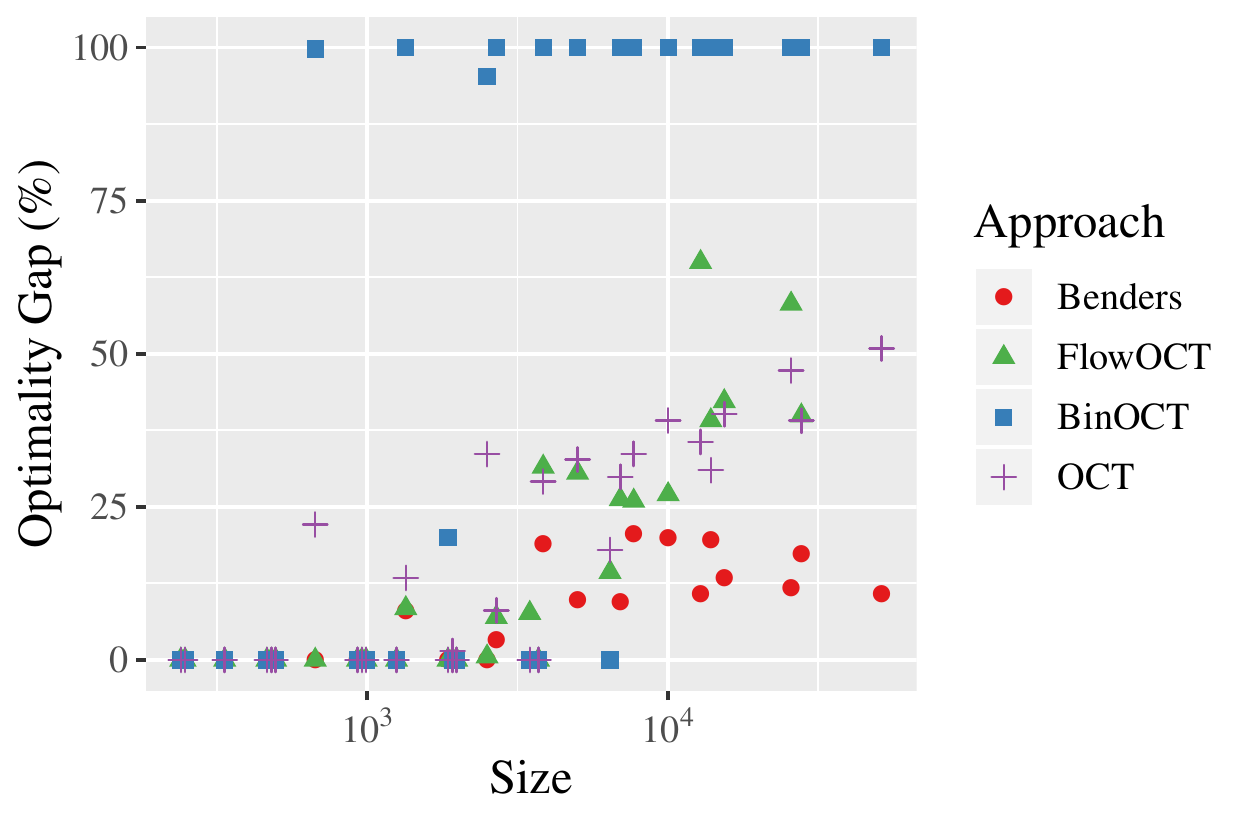}}
    \caption{Summary of optimization performance.}
    \label{fig:performace-gap}
\end{figure*}


\paragraph{In-Sample (Optimization) Performance.} Figure~\ref{fig:performace-gap} summarizes the in-sample performance, i.e., how good the methods are at solving the optimization problems. Detailed results are provided in the online companion B. From Figure~\ref{fig:performace-gap}(a), we observe that for $\lambda=0$, \texttt{BinOCT} is able to solve 79 instances within the time limit (and outperforms \texttt{OCT}), but \texttt{Benders} solves the same quantity of instances in only 140 seconds, \emph{resulting in a $30\times$ speedup}. Similarly, from Figure~\ref{fig:performace-gap}(b), it can be seen that for $\lambda>0$, \texttt{OCT} is able to solve 666 instances within the time limit\footnote{\texttt{BinOCT} does not include the option to have a regularization parameter, and is omitted.}, while \texttt{Benders} requires only 70 seconds to do so, \emph{resulting in a $50\times$ speedup}. Finally, Figure~\ref{fig:performace-gap}(c) shows the optimality gaps proven as a function of the dimension. We observe that all methods result in a gap of $0\%$ in small instances. As the dimension increases, BinOCT (which relies on weak formulations but fast enumeration) yields 100\% optimality gaps in most cases. \texttt{OCT} and \texttt{BinOCT} prove better gaps, but the performance degrades substantially as the dimension increases. \texttt{Benders} results in the best performance, proving optimality gaps of 20\% or less regardless of dimension.

\paragraph{Out-of-Sample (Statistical) Performance.} Table~\ref{tab:out_of_sample} reports the out-of-sample accuracy after cross-validation. Each row represents the average over the five samples. We observe that the better optimization performance translates to superior statistical properties as well: \texttt{OCT} is the best method in two instances (excluding ties), \texttt{BinOCT} in six, while the new formulations \texttt{FlowOCT} and \texttt{Benders} are better in 13 (of which \texttt{Benders} accounts for 10, and is second after \texttt{FlowOCT} in an additional two).

\begin{table}[!ht]
\begin{center}
\caption{Out of sample accuracy}
\label{tab:out_of_sample}
\setlength{\tabcolsep}{2pt}

\begin{tabular}{lc|c|c|c|c}
\hline
Dataset & Depth & OCT & BinOCT & FlowOCT & Benders\\
\hline
monk3           & 2     & \textbf{92.3} & \textbf{92.3} & \textbf{92.3} & \textbf{92.3} \\
monk3           & 3     & 83.2          & \textbf{91}   & \textbf{91}   & \textbf{91}   \\
monk3           & 4     & 91            & 85.2          & \textbf{92.3} & 91            \\
monk3           & 5     & 87.1          & 87.7          & \textbf{92.3} & 91.6          \\
monk1           & 2     & 71            & \textbf{72.3} & \textbf{72.3} & 71            \\
monk1           & 3     & \textbf{83.2} & 82.6          & 81.3          & 81.3          \\
monk1           & 4     & \textbf{100}  & 99.4          & \textbf{100}  & \textbf{100}  \\
monk1           & 5     & 93.5          & 96.8          & \textbf{100}  & \textbf{100}  \\
monk2           & 2     & \textbf{56.7} & 49.8          & \textbf{56.7} & \textbf{56.7} \\
monk2           & 3     & 62.3          & 58.1          & \textbf{63.7} & 63.3          \\
monk2           & 4     & 59.5          & 60.5          & 58.6          & \textbf{64.2} \\
monk2           & 5     & \textbf{63.3} & 55.8          & 62.3          & 61.9          \\
house-votes-84  & 2     & 79.3          & 96.2          & \textbf{97.2} & \textbf{97.2} \\
house-votes-84  & 3     & \textbf{97.2} & 94.1          & \textbf{97.2} & \textbf{97.2} \\
house-votes-84  & 4     & \textbf{96.9} & 94.8          & \textbf{96.9} & 95.5          \\
house-votes-84  & 5     & 95.2          & 93.1          & 96.9          & \textbf{97.2} \\
balance-scale   & 2     & \textbf{68.7} & 67.9          & \textbf{68.7} & \textbf{68.7} \\
balance-scale   & 3     & 69            & \textbf{71.5} & 69.8          & 71            \\
balance-scale   & 4     & 68.5          & \textbf{73.9} & 73.2          & 71.7          \\
balance-scale   & 5     & 65.7          & 75.3          & 71.6          & \textbf{76.8} \\
tic-tac-toe     & 2     & \textbf{66.7} & 65.9          & \textbf{66.7} & \textbf{66.7} \\
tic-tac-toe     & 3     & 68.1          & 72.2          & 68.5          & \textbf{72.6} \\
tic-tac-toe     & 4     & 70.4          & \textbf{80.3} & 68.7          & 77.1          \\
tic-tac-toe     & 5     & 69.7          & 78.9          & 66.3          & \textbf{79.3} \\
car\_evaluation & 2     & \textbf{76.5} & \textbf{76.5} & \textbf{76.5} & \textbf{76.5} \\
car\_evaluation & 3     & 73.3          & 78.4          & 76.7          & \textbf{79.1} \\
car\_evaluation & 4     & 75.2          & \textbf{80.3} & 71.6          & 79.7          \\
car\_evaluation & 5     & 74.8          & \textbf{81.3} & 61.6          & 80.5          \\
kr-vs-kp        & 2     & 73.7          & \textbf{87.2} & 70.5          & \textbf{87.2} \\
kr-vs-kp        & 3     & 69.3          & 87.8          & 61.2          & \textbf{89.9} \\
kr-vs-kp        & 4     & 64.7          & 90.8          & 54.3          & \textbf{91}   \\
kr-vs-kp        & 5     & 62.7          & \textbf{87.1} & 45.8          & 86.7 \\
\hline
\end{tabular}
\end{center}
\end{table}

\phantom{ bla }

\newpage
\phantom{bla}
\newpage

\bibliography{bib}

\newpage

\appendix

\section{Proof of Theorem~\ref{theo:facet}}
\label{appendix_sec:polyhedral}

\newpv{The proof proceeds in three steps. We fix $i \in \sets I$. We derive the specific structure $(p,q) \in \sets P$ of the cuts associated with datapoint~$i$ generated by our procedure. We then provide $|\mathcal N \times \mathcal F| + |\mathcal L \times \mathcal K| + |\mathcal I|$ points that lie in $\textrm{conv}(\sets H_{\leq})$ and at each of which the cut generated holds with equality. Since the choice of $i \in \sets I$ is arbitrary and since the cuts generated by our procedure are valid (by construction), this will conclude the proof.}

\newpv{To minimize notational overhead, we assume throughout this proof that $\lambda = 0$. In this case, an optimal solution to the master problem where $\sum_{k\in \sets K} w_{nk}=0$ for all $n \in \sets N$ can always be obtained. 
} Given a set $A$ and a point $a\in A$, we use $A\setminus a$ as a shorthand for $A\setminus\{a\}$. Finally, \newpv{we} let \newpv{$e_{ij}=1$} be a vector (whose dimensions will be clear from the context) with a $1$ in coordinates $(i,j)$ and $0$ elsewhere.

\newpv{Fix $i \in \sets I$}. Let $(\bar b,\bar w,\bar g)$ be optimal in the (relaxed) master problem and assume $\sum_{k\in \sets K} \bar w_{nk}=0$ for all $n \in \sets N$. \newpv{Given $j\in \sets I$,} let $n(j)\in \sets L$ be the leaf of the tree defined by $(\bar b,\bar w)$ that datapoint~$j$ is assigned to. \newpv{Given $n\in \sets N$, let $f(n)\in \sets F$ be the feature selected for testing at node~$n$ under $(\bar b,\bar w)$, i.e., $\bar b_{nf(n)}=1$.}

\newpv{We now derive the structure of the cuts~\eqref{eq:mip_MP2_b} generated by Algorithm~\ref{alg:cut} (see also the proof of Proposition~\ref{prop:algorithm}) when $(\bar b,\bar w,\bar g)$ is input. A minimum cut is returned by Algorithm~\ref{alg:cut} if and only if~$s$ and~$t$ belong to different connected components in the tree induced by $(\bar b,\bar w)$. Under this assumption, since $\sum_{k\in \sets K} \bar w_{nk}=0$ for all $n \in \sets N$, the connected component $\sets S$ constructed in Algorithm~\ref{alg:cut} forms a path from~$s$ to~$n_d = n(i) \in \sets L$, i.e., $\sets S=\{s,n_1,n_2,\ldots,n_d\}$. The minimum cut $q$ obtained from Algorithm~\ref{alg:cut} then corresponds to the arcs adjacent to nodes in~$\sets S$ that do not belong to the path formed by~$\sets S$. Therefore, $q_{s,1} = 0$, $q_{n,t}=1$ iff $n=n(i)$, and for each $n \in \sets N$,
$$
q_{n,\ell(n)}=1 \;\; \text{  iff  } \;\; n \in p \;\; \text{  and  } \;\; \sum_{f \in \sets F : x_f^i = 1} \bar b_{nf} = 1, \text{ and}
$$
$$
q_{n,r(n)}=1 \;\; \text{  iff  } \;\; n \in p \;\; \text{  and  } \;\; \sum_{f \in \sets F : x_f^i = 0} \bar b_{nf} = 1.
$$
Therefore, the cut~\eqref{eq:mip_MP2_b} returned by Algorithm~\ref{alg:cut} reads
\begin{equation}
    g_i \; \leq \; w_{n(i)y^i} +  \sum_{n \in \sets S} \sum_{ \begin{smallmatrix} f \in \sets F : \\  x_f^i \neq x_{f(n)}^i \end{smallmatrix} } b_{nf}.
    \label{eq:facet}
\end{equation}
}

Next, we give $|\sets N\times \sets F|+|\sets L\times \sets K|+|\sets I|$ affinely independent points in \newpv{$\sets H_{\leq}$} \newpv{for} which~\eqref{eq:facet} holds with equality. Given a vector \newpv{$b\in \{0,1\}^{|\sets N| \times |\sets F|}$}, we \newpv{let $b_{\sets S}$ (resp.\ $b_{\sets N\setminus \sets S}$) collect those elements of~$b$ whose first index} \newpv{$n\in \sets S$ (resp.\ $n \notin \sets S$)}. We now describe the points, which are also summarized in Table~\ref{tab:affine}.

\begin{table*}[h!]
\begin{center}
\small{
	\caption{\newpv{Companion table for the proof of Theorem~\ref{theo:facet}:} list of affinely independent points \newpv{that live on the cut generated by inputing $i \in \sets I$ and $(\bar b,\bar w,\bar g)$ in Algorithm~\ref{alg:cut}.}}
	\label{tab:affine}
	\begin{tabular}{c|l|r c c c c }
		\hline
		\multirow{2}{*}{\#} & \multirow{2}{*}{condition} \hfill dim$=$ & &$|\sets S|\times |\sets F|$ & $|\sets N\setminus \sets S|\times |\sets F|$&$|\sets L|\times |\sets K|$& $|\sets I|$\\
		& \hfill sol$=$ & & $b_{\sets S}$ & $b_{\sets N\setminus \sets S}$ & $w$&$g$\\
		\hline
		&&&&&&\\
		1 & ``baseline'' point  && $\bar b_\sets S$ & 0 & 0 & 0 \\
		&&&&&&\\
		2 & $n\in \sets L,k\in \sets K\setminus y^i$&& $\bar b_{\sets S}$ & 0 &  $e_{nk}$ & 0 \\
		3 & $n\in \sets L\setminus n(i)$&  & $\bar b_{\sets S}$ & 0 &  $e_{ny^i}$  & 0\\
		4 & \newpv{$n=n(i)$} &  & $\bar b_{\sets S}$ & 0 & $e_{n(i) y^i}$ & $e_{i}$ \\
		&&&&&&\\
		5 & $n\in \sets N\setminus \sets S,f\in \sets F$&& $\bar b_{\sets S}$ & $e_{nf}$ & 0 & 0 \\
		&&&&&&\\
		6 & $n\in \sets S$& & $\bar b_{\sets S}-e_{nf(n)}$ & 0 & 0 & 0 \\
		7 & $n\in \sets S,f\in \sets F:f\neq f(n),x_f^i=x_{f(n)}^i$&  & $\bar b_{\sets S}-e_{nf(n)}+e_{nf}$ & 0 & 0 & 0\\
		8 & $n\in \sets S,f\in \sets F:f\neq f(n),x_f^i\neq x_{f(n)}^i$&  & $\bar b_{\sets S}-e_{nf(n)}+e_{nf}$ & 0 & $\displaystyle\sum_{n\in \sets L: n\neq n(i)}e_{ny^i}$ & $e_i$\\
		&&&&&&\\
		9 & $j\in \sets I\setminus i: y^j\neq y^i$ && $\bar b_{\sets S}$ & $\bar b_{\sets N\setminus \sets S}$ & $e_{n(j)y^j}$ & $e_j$ \\
		10 & $j\in \sets I\setminus i: y^j=y^i$, $n(j)\neq n(i)$& & $\bar b_{\sets S}$ & $\bar b_{\sets N\setminus \sets S}$ & $e_{n(j)y^j}$ & $e_j$\\
		11 & $j\in \sets I\setminus i: y^j=y^i$, $n(j)= n(i)$& & $\bar b_{\sets S}$ & $\bar b_{\sets N\setminus \sets S}$ & $e_{n(i)y^i}$ & $e_i+e_j$\\
		\hline
	\end{tabular}
	}
	\end{center}
\end{table*}

\begin{description}
	\item[1] \hspace{0.25cm} \newpv{One point that is a ``baseline'' point; all other points are variants of it. It is given by} \newpv{$b_{\sets S}=\bar b_{\sets S}$}, \newpv{$b_{\sets N\setminus \sets S} =0$}, $w=0$ and $g=0$ \newpv{and corresponds to} selecting the features to \newpv{branch on} according to~$\bar b$ for nodes \newpv{in $\sets S$} and setting all remaining variables to $0$\newpv{. The baseline point belongs to $\sets H_\leq$} and constraint~\eqref{eq:facet} \newpv{is active at this point}.
	\item[2-4] $|\sets L|\times |\sets K|$ points \newpv{obtained from the baseline point by varying the $w$ coordinates and adjusting $g$ as necessary to ensure~\eqref{eq:facet} remains active:} \newpv{\textbf{2:}~$|\sets L|\times (|\sets K|-1)$ points, each associated with a leaf $n \in \sets L$ and class $k \in \sets K : k\neq y^i$, where the label of leaf $n$ is changed to $k$}. \textbf{3:}~\newpv{$|\sets L|-1$ points, each associated with a leaf $n \in \sets L : n \neq n(i)$, where the class label of~$n$ is changed to~$y^i$}. \textbf{4:}~\newpv{One point where the} class label of leaf \newpv{$n(i)$} is set to~$y^i$, allowing for correct classification of datapoint~$i$; in this case, the value of the rhs of \eqref{eq:facet} is 1, and we set $g^i=1$ to ensure \newpv{the cut~\eqref{eq:facet} remains active.}
	\item[5] \hspace{0.25cm} \newpv{$|\sets N\setminus \sets S|\times |\sets F|$ points obtained from the baseline point by varying the~$b_{\sets N \backslash \sets S}$ coordinates. Each point is associated with a node $n\in \sets N \backslash \sets S$ and feature $f \in \sets F$ and is obtained by changing the decision to branch on feature $f$ and node $n$ to 1}. As those branching decisions do not impact \newpv{the routing of} datapoint~$i$ \newpv{the value of the rhs of inequality~\eqref{eq:facet} remains unchanged and the inequality stays active}.
	\item[6-8] \newpv{$| \sets S |\times |\sets F|$ points, obtained from the baseline point by varying the $b_\sets S$ coordinates and adjusting~$w$ and~$g$ as necessary to guarantee feasibility of the resulting point and to ensure that~\eqref{eq:facet} stays active.}  \textbf{6:} \newpv{$| \sets S |$ points, each associated with a node $n \in \sets S$ obtained by not branching on feature $f(n)$ at node $n$ (nor on any other feature)}, resulting in a ``dead-end'' node. \newpv{The value of the rhs of~\eqref{eq:facet} is unchanged in this case and the inequality remains active.} \textbf{7-8:} \newpv{$| \sets S |$ points, each associated with a node $n \in \sets S$ and feature $f \neq f(n)$. }
	\textbf{7:}~If the branching decision $f(n)$ at node $n$ is replaced with a branching decision that results in the same path for datapoint~$i$\newpv{, i.e., if $x_f^i=x_{f(n)}^i$}, it is possible to swap those decisions without affecting the \newpv{value} of the rhs in inequality~\eqref{eq:facet}. \textbf{8:}~If a feature that causes~$i$ to change paths is chosen for branching\newpv{, i.e., if $x_f^i \neq x_{f(n)}^i$}, then the \newpv{value of the} rhs of~\eqref{eq:facet} is increased by 1, and we set $g^i=1$ to ensure the inequality \newpv{remains active}; to guarantee feasibility of the resulting point, \newpv{we} label each leaf node except for \newpv{$n(i)$} with the class $y^i$, which does not affect inequality~\eqref{eq:facet}. 
	\item[9-11] \newpv{$|I|-1$ points, obtained from the baseline point by letting~$b_{\sets N \backslash \sets S}=\bar b_{\sets N \backslash \sets S}$ and adjusting~$w$ and~$g$ as necessary.} Each point is associated with a datapoint~$j\in \sets I\setminus i$ \newpv{which we allow to be correctly} classified. \textbf{9:}~\newpv{If datapoint~$j$ has a different class than datapoint~$i$ ($y^j \neq y^i$), we label the leaf node where~$j$ is routed to with the class of~$j$}, i.e., $w_{n(j)y^j}=1$. \newpv{The value of the rhs of~\eqref{eq:facet} is unaffected the inequality remains active}.
	\textbf{10:}~\newpv{If datapoint~$j$ has the same class as datapoint~$i$} but \newpv{is} routed to \newpv{a} different leaf \newpv{than $i$, an argument paralleling that in ~\textbf{9} can be made}. \textbf{11:}~\newpv{If datapoint~$j$ has the same class as datapoint~$i$} and \newpv{is} routed to the same leaf \newpv{$n(i)$}, we label \newpv{$n(i)$} with the class of $y^i=y^j$ and set $g^j=1$; the \newpv{value of the rhs} of~\eqref{eq:facet} increases by~$1$. \newpv{Thus,} we set also correctly classify datapoint $i$ by setting $g^i=1$ to ensure that \eqref{eq:facet} \newpv{is active}.
\end{description}
The $|\sets N\times \sets F|+|\sets L\times \sets K|+|\sets I|$ \newpv{points constructed above, see also Table~\ref{tab:affine}, are} affinely independent. \newpv{Indeed, each differs from the previously introduced points in at least one coordinate. All these points also belong to $\sets H_{\leq}$. This concludes the proof.}

\section{OCT}
\label{appendix_sec:OCT}
In this section, we provide a simplified version of the formulation of \citet{bertsimas2017optimal} specialized to the case of binary data.

\begin{figure}[ht]
\vskip 0.2in
\begin{center}
\centerline{\includegraphics[width=0.5\columnwidth]{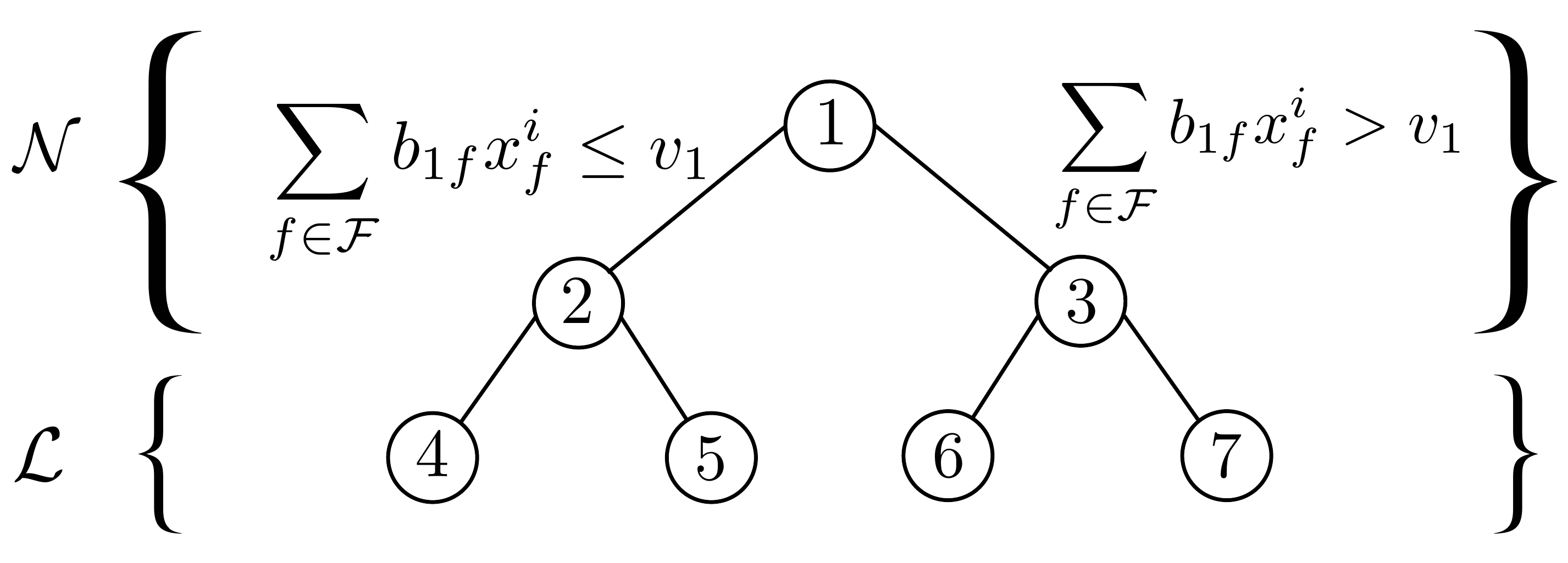}}
\caption{A classification tree of depth 2}
\label{fig:sample_tree_B}
\end{center}
\vskip -0.2in
\end{figure}

We start with introducing the notation that is used for the formulation. Let $\sets N$ and $\sets L$ denote the sets of all internal and leaf nodes in the tree structure. For each node $n \in \sets N \cup \sets L \backslash \{1\}$, $a(n)$ refers to the direct ancestor of node $n$. $\sets {AL}(n)$ is the set of ancestors of $n$ whose left branch has been followed on the path from the root node to $n$, and similarly $\sets {AR}(n)$ is the set of right-branch ancestors, such that $\sets A(n) = \sets{AL}(n) \cup \sets{AR}(n)$

Let $b_{nf}$ be a binary decision variable where $b_{nf}=1$ iff at node n, feature f is branched upon.
For each datapoint $i$ at node $n \in \sets N$ a test $ \sum_{f \in \sets F}b_{nf}x^i_f < v_n$ is performed where $v_n \in \mathbb{R}$ is a decision variable represeinting the cut-off value of the test. If datapoint $i$ passes the test it follows the left branch otherwise it follows the right one. Let $d_n=1$ iff node $n$ applies a split, to allow having this option not to split at a node.
To track each datapoint $i$ through the tree, the decision variable $ \zeta_{a(n),n}^i$ is introduced where $ \zeta_{a(n),n}^i = 1$ iff datapoint $i$ is in node $n$.

Let $P_{nk}$ to be the number of datapoints of class $k$ assigned to leaf node $n$ and $P_n$ to be the total number of datapoints in leaf node $n$.
Let $w_{nk}$ denote the prediction of each leaf node $n$, where $w_{nk}=1$ iff the predicted label of node $n$ is $k \in \sets K$.
At the end, let $L_n$ denote the number of missclassified datapoints at node $n$.

\begin{subequations}
\begin{align}
\max \;\;&  \displaystyle (1-\lambda)\left(|\sets I|-\sum_{n \in \sets L}L_n\right) - \lambda \sum_{n \in \sets N}d_n \label{eq:OCT_a}\\
\text{s.t.} \; \; & \displaystyle L_n \geq P_n - P_{nk} - |\sets I|(1-w_{nk}) & \vspace{-5cm} \forall k \in \sets K, n \in \sets L \label{eq:OCT_b}\\
& \displaystyle L_n \leq P_n - P_{nk} + |\sets I|w_{nk} & \vspace{-5cm}  \forall k \in \sets K, n \in \sets L \label{eq:OCT_c}\\
&  \displaystyle P_{nk}= \sum_{\begin{smallmatrix} i\in \sets I:\\ y^i=k \end{smallmatrix}} \zeta_{a(n),n}^i & \vspace{-5cm} \forall k \in \sets K, n \in \sets L \label{eq:OCT_d}\\
&  P_n=\sum_{i \in \sets I} \zeta_{a(n),n}^i & \vspace{-5cm}  \forall n \in \sets L \label{eq:OCT_e}\\
&  l_n= \sum_{k \in \sets K}w_{nk} & \vspace{-5cm} \forall n \in \sets L \label{eq:OCT_f}\\
&   \zeta_{a(n),n}^i \leq l_n & \vspace{-5cm} \forall n \in \sets L \label{eq:OCT_g}\\
&  \sum_{n \in \sets L} \zeta_{a(n),n}^i = 1 & \vspace{-5cm} \forall i \in \sets I \label{eq:OCT_h}\\
&    \sum_{f \in \sets F}b_{mf}x^i_f \geq v_m + \zeta_{a(n),n}^i -1 & \vspace{-5cm} \forall i \in \sets I, n \in \sets L, m \in \sets {AR}(n)  \label{eq:OCT_i}\\
&   \sum_{f \in \sets F}b_{mf}x^i_f \leq v_m - 2 \zeta_{a(n),n}^i +1 & \vspace{-5cm} \forall i \in \sets I, n \in \sets L, m \in \sets {AL}(n)  \label{eq:OCT_j}\\
&  \sum_{f \in \sets F}b_{nf} = d_n & \vspace{-5cm} \forall n \in \sets N \label{eq:OCT_k}\\
&   0 \leq v_n \leq d_n & \vspace{-5cm} \forall n \in \sets N \label{eq:OCT_l}\\
&   d_n\leq d_{a(n)} & \vspace{-5cm}  \forall n \in \sets N\backslash\{1\} \label{eq:OCT_m}\\
&   z^i_n, l_n \in \{0,1\}  & \vspace{-5cm} \forall i \in \sets I, n \in \sets L \label{eq:OCT_n}\\
&   b_{nf}, d_n \in \{0,1\} & \vspace{-5cm} \forall f \in \sets F, n \in \sets N \label{eq:OCT_o},
\end{align}
\label{eq:OCT}
\end{subequations}

where $\lambda \in [0,1]$ is a regularization term. The objective~\eqref{eq:OCT_a} maximizes the total number of correctly classified datapoints $|\sets I|-\sum_{n \in \sets L}L_n$ while minimizing the number of splits $\sum_{n \in \sets N}d_n $.
Constraints~\eqref{eq:OCT_b} and~\eqref{eq:OCT_c} defines the number of missclassified datapoints at each node $n$. Constraints~\eqref{eq:OCT_d} and~\eqref{eq:OCT_e} give the definitions of $P_{nk}$ and $P_n$ respectively.
constraints~\eqref{eq:OCT_f}-~\eqref{eq:OCT_g}, enforce that if a leaf $n$ does not have an assigned class label, no datapoint should end up at that leaf. Constraint~\eqref{eq:OCT_h} makes sure that each datapoint $i$ is assigned to exactly one of the leaf nodes. 
Constraint~\eqref{eq:OCT_i} implies that if datapoint $i$ is assigned to node $n$, it should take the right branch for all ancestors of $n$ belonging to $\sets {AR}(n)$.  Respectively, constraint~\eqref{eq:OCT_j} implies that if datapoint $i$ is assigned to node $n$, it should take the left branch for all ancestors of $n$ belonging to $\sets {AL}(n)$.
Constraint~\eqref{eq:OCT_k} enforces that if node $n$ splits, it should split on exatcly one of the features $f \in \sets F$. 
Constraint~\eqref{eq:OCT_l} implies that if a node does not apply a split, all datapoints going through this node would take the right branch. At the end constraint~\eqref{eq:OCT_m} makes sure that if node $n$ does not split, none of its descendants cannot split.

In the main formulation of~\citet{bertsimas2017optimal}, they have parameter $N_{\text{min}}$ which denotes the minimum number of points at each leaf. We set this parameter to zero as we do not have a similar notion in our formulation. 

\section{Comparison with OCT}
\label{appendix_sec:oct_comparison}
In formulation~\eqref{eq:OCT}, $v_n$ can be fixed to $d_n$ for all nodes (for the case of binary data) leading to the simplified formulation

\begin{align*}
\max \;\;&  \displaystyle (1-\lambda)\left(|\sets I|-\sum_{n \in \sets L}L_n\right) - \lambda \sum_{n \in \sets N}d_n \\
\text{s.t.}\;\;& \displaystyle L_n \geq P_n - P_{nk} - |\sets I|(1-w_{nk}) & \vspace{-5cm}  \forall k \in \sets K, n \in \sets L \\
& \displaystyle L_n \leq P_n - P_{nk} + |\sets I|w_{nk} & \vspace{-5cm} \forall k \in \sets K, n \in \sets L \\
&  \displaystyle P_{nk}= \sum_{\begin{smallmatrix} i\in \sets I:\\ y^i=k \end{smallmatrix}} \zeta_{a(n),n}^i & \vspace{-5cm} \forall k \in \sets K, n \in \sets L \\
&  P_n=\sum_{i \in \sets I} \zeta_{a(n),n}^i & \vspace{-5cm}  \forall n \in \sets L \\
&  l_n= \sum_{k \in \sets K}w_{nk} & \vspace{-5cm}  \forall n \in \sets L \\
&   \zeta_{a(n),n}^i \leq l_n & \vspace{-5cm} \forall n \in \sets L \\
&  \sum_{n \in \sets L} \zeta_{a(n),n}^i = 1 & \vspace{-5cm}  \forall i \in \sets I \\
&    \sum_{f \in \sets F:x_f^i=1}b_{mf} \geq d_m + \zeta_{a(n),n}^i -1 & \vspace{-5cm}   \forall i \in \sets I, n \in \sets L, m \in \sets {AR}(n)  \\
&   \sum_{f \in \sets F:x_f^i=1}b_{mf} \leq d_m - 2 \zeta_{a(n),n}^i +1 & \vspace{-5cm} \forall i \in \sets I, n \in \sets L, m \in \sets {AL}(n)  \\
&  \sum_{f \in \sets F}b_{nf} = d_n & \vspace{-5cm} \forall n \in \sets N \\
&   d_n\leq d_{a(n)} & \vspace{-5cm} \forall n \in \sets N\backslash\{1\} \\
&   z^i_n, l_n \in \{0,1\} & \vspace{-5cm} \forall i \in \sets I, n \in \sets L \\
&   b_{nf}, d_n \in \{0,1\} & \vspace{-5cm} \forall f \in \sets F, n \in \sets N ,
\end{align*}

Note that fixing $v_n$ is a simplification due to the assumption of binary data, rather than an actual strengthening. Moreover, note that OCT and FlowOCT have different conventions for nodes where branching is not performed: in FlowOCT, a label (encoded by $w_{nk}$) is directly assigned to that node, while in OCT all points go right by convention. This different convention creates a slight change in the feasible region of both formulations. To be able to directly compare the formulations, we consider the case of ``full" trees where branching is performed at all internal nodes $\sets N$. For FlowOCT formulation, this corresponds to setting $w_{nk}=0$ for all $n\in \sets N$, $k\in \sets K$, while for OCT it corresponds to setting $d_n=1$ for all $n\in \sets S$. Moreover, using the identity $\sum_{f \in \sets F:x_f^i=1}b_{mf}=1-\sum_{f \in \sets F:x_f^i=0}b_{mf}$ and and noting that $l_n$ can be fixed to $1$ in the formulation, we obtain the simplified OCT formulation

\begin{subequations}\label{eq:OCT2}
\begin{align}
\max \;\;&  \displaystyle \left(|\sets I|-\sum_{n \in \sets L}L_n\right) \\
\text{s.t. }\;\;&  \displaystyle L_n \geq P_n - P_{nk} - |\sets I|(1-w_{nk}) \quad&\hspace{-2cm}  \forall k \in \sets K, n \in \sets L \hfill\label{eq:OCT2_first}\\
& \displaystyle L_n \leq P_n - P_{nk} + |\sets I|w_{nk} \quad &\hspace{-2cm} \forall k \in \sets K, n \in \sets L \hfill \label{eq:OCT2_ub}\\
&  \displaystyle P_{nk}= \sum_{\begin{smallmatrix} i\in \sets I:\\ y^i=k \end{smallmatrix}} \zeta_{a(n),n}^i \quad &\hspace{-2cm}\forall k \in \sets K, n \in \sets L \hfill\label{eq:OCT2_P1}\\
&  P_n=\sum_{i \in \sets I} \zeta_{a(n),n}^i \quad &\hspace{-2cm}\forall n \in \sets L \hfill\label{eq:OCT2_P2}\\
&   \sum_{k \in \sets K}w_{nk}=1 \quad &\hspace{-2cm}\forall n \in \sets L \hfill\label{eq:OCT2_label}\\
&  \sum_{n \in \sets L} \zeta_{a(n),n}^i = 1 \quad  &\hspace{-2cm} \forall i \in \sets I \hfill\label{eq:OCT2_assign}\\
&    \sum_{f \in \sets F:x_f^i=1}b_{mf} \geq  \zeta_{a(n),n}^i &\hspace{-2cm}\forall i \in \sets I, n \in \sets L, m \in \sets {AR}(n) \label{eq:OCT2_right} \hfill\\
&  \sum_{f \in \sets F:x_f^i=0}b_{mf} \geq  2\zeta_{a(n),n}^i-1 &\hspace{-2cm}\forall i \in \sets I, n \in \sets L, m \in \sets {AL}(n)  \hfill\label{eq:OCT2_bound}\\
&  \sum_{f \in \sets F}b_{nf} = 1  \quad  &\hspace{-2cm}\forall n \in \sets N \hfill\label{eq:OCT2_branch}\\
&   \zeta_{a(n),n}^i \in \{0,1\}  \quad&\hspace{-2cm} \forall i \in \sets I, n \in \sets L \hfill\\
&   b_{nf}\in \{0,1\} \quad & \hspace{-10cm}\forall f \in \sets F, n \in \sets N.\label{eq:OCT2_last}
\end{align}
\end{subequations}

\subsection{Strengthening}\label{sec:strength}

We now show how formulation \eqref{eq:OCT2} can be strengthened. Observe that the validity of the steps below is guaranteed by the validity of FlowOCT, thus we do not focus on validity below. 

\paragraph{Bound tightening for \eqref{eq:OCT2_bound}} Adding the quantity $1-\zeta_{a(n),n}^i\geq 0$ to the right hand side of \eqref{eq:OCT2_bound}, we obtain the stronger constraints
\begin{equation}\label{eq:OCT2_left}
    \sum_{f \in \sets F:x_f^i=0}b_{mf} \geq  \zeta_{a(n),n}^i \; \forall i \in \sets I, n \in \sets L, m \in \sets {AL}(n).
\end{equation}

\paragraph{Improved branching constraints} 
Constraints \eqref{eq:OCT2_right} can be strengthened to 
\begin{equation}\label{eq:OCT2_right_new}
     \sum_{f \in \sets F:x_f^i=1}b_{mf} \geq  \sum_{n\in \sets L: m\in \sets{ AR}(n)}\zeta_{a(n),n}^i \; \forall i \in \sets I, m \in \sets N.
\end{equation}
Observe that constraints \eqref{eq:OCT2_right_new}, in addition to being stronger than \eqref{eq:OCT2_right}, also reduce the number of constraints require to represent the LP relaxation. Similarly, constraint \eqref{eq:OCT2_left} can be further improved to 
\begin{equation}\label{eq:OCT2_left_new}
     \sum_{f \in \sets F:x_f^i=0}b_{mf} \geq  \sum_{n\in \sets L: m\in \sets{ AL}(n)}\zeta_{a(n),n}^i \; \forall i \in \sets I, m \in \sets N.
\end{equation}

\paragraph{Improved missclassification formulation} 

Define for all $i\in \sets I$ and $n\in \sets L$ additional variables $z_{a(n),n}^i\leq \zeta_{a(n),n}w_{n,y^i}$. Note that $z_{a(n),n}^i=1$ implies that datapoint $i$ is routed to leaf $n$ ($\zeta_{a(n),n}^i=1$) and the class of $i$ is assigned to $n$ ($w_{ny^i}=1$), hence $z_{a(n),n}^i=1$ only if $i$ is correctly classified at leaf $n$. Upper bounds of $z_{a(n),n}^i=1$ can be imposed via the linear constraints 
\begin{equation}\label{eq:linearization}
  z_{a(n),n}^i\leq  \zeta_{a(n),n}^i,\; z_{a(n),n}^i\leq w_{ny^i}\quad \forall n\in\sets L, i\in \sets I. 
\end{equation}
In addition, since $L_n$ corresponds to the number of missclassified points at leaf $n\in \sets L$ and $\sum_{n\in \sets L}L_n$, we find that constraints 
\begin{equation}\label{eq:missclass}
    L_n\geq \sum_{i\in \sets I}(\zeta_{n,a(n)}^i -z_{n,a(n)}^i).
\end{equation}
Note that constraints \eqref{eq:missclass} and \eqref{eq:OCT2_assign} imply that 
\begin{equation}\label{eq:totalMissclass}
    \sum_{n\in \sets L}L_n\geq |\sets I| -\sum_{i\in \sets I}\sum_{n\in \sets L}z_{n,a(n)}^i.
\end{equation}

\subsection{Simplification}\label{sec:simple}
The linear programming relaxation of the formulation obtained in \S\ref{sec:strength}, given by constraints \eqref{eq:OCT2_first}-\eqref{eq:OCT2_assign}, \eqref{eq:OCT2_branch}-\eqref{eq:OCT2_last}, \eqref{eq:OCT2_right_new}, \eqref{eq:OCT2_left_new}, \eqref{eq:linearization} and \eqref{eq:missclass}, is certainly stronger than the relaxation of OCT, as either constraints where tightened or additional constraints were added. We now show how the resulting formulation can be simplified without loss of relaxation quality, ultimately obtaining FlowOCT.  

\paragraph{Upper bound on missclassification} Variable $L_n$ has a negative objective coefficient and only appears on constraints \eqref{eq:OCT2_first}-\eqref{eq:OCT2_ub} and \eqref{eq:missclass}, it will always be set to a lower bound. Therefore, constraint \eqref{eq:OCT2_ub} is redundant and can be dropped without affecting the relaxation of the problem.

\paragraph{Lower bound on missclassification} Substituting variables according to \eqref{eq:OCT2_P1} and \eqref{eq:OCT2_P2}, we find that for a given $k\in \sets K$ and $n\in \sets L$, \eqref{eq:OCT2_first} is equivalent to 
\begin{align*}
    &L_n \geq \sum_{i \in \sets I} \zeta_{a(n),n}^i  - \sum_{\begin{smallmatrix} i\in \sets I:\\ y^i=k \end{smallmatrix}} \zeta_{a(n),n}^i - |\sets I|(1-w_{nk})\\
    \Leftrightarrow\;&L_n \geq \sum_{\substack{i \in \sets I\\y_i=k}}(w_{nk}-1)  +\sum_{\substack{i \in \sets I\\y^i\neq k}}(\zeta_{a(n),n}^i-1+w_{nk}).
\end{align*}
Observe that $w_{nk}-1\leq 0 \leq \zeta_{a(n),n}^i-z_{a(n),n}^i$. Moreover, we also have that for any $i\in \sets I$ and $k\in \sets K\setminus\{y^i\}$, \begin{equation}\label{eq:ineqs}z_{a(n),n}^i\leq w_{ny^i}\leq 1-w_{nk},\end{equation} where the first inequality follows from \eqref{eq:linearization} and the second inequality follows from \eqref{eq:OCT2_label}. Therefore from \eqref{eq:ineqs} we conclude that $\zeta_{a(n),n}^i-1+w_{nk}\leq\zeta_{a(n),n}^i-z_{a(n),n}^i$ and inequalities \eqref{eq:missclass} dominate inequalities \eqref{eq:OCT2_first}. Since inequalities \eqref{eq:OCT2_P1}-\eqref{eq:OCT2_P2} only appeared in inequalities \eqref{eq:OCT2_first}-\eqref{eq:OCT2_ub}, which where shown to be redundant, they can be dropped as well. Finally, as inequalities \eqref{eq:missclass} define the unique the lower bounds of $L_n$ in the simplified formulation, they can be changed to equalities without loss of generalities, and the objective can be updated according to \eqref{eq:totalMissclass}. After all the changes outlined so far, the formulation reduces to 
\begin{subequations}\label{eq:OCT2.0}
\begin{align}
\max \;\;&  \displaystyle \sum_{i\in \sets I}\sum_{n\in \sets L}z_{n,a(n)}^i\\
\text{s.t. }\;\;& \sum_{k \in \sets K}w_{nk}=1 \quad &\hspace{-2cm}\forall n \in \sets L \hfill\label{eq:OCT2.0_label}\\
&  \sum_{n \in \sets L} \zeta_{a(n),n}^i = 1 \quad  &\hspace{-2cm} \forall i \in \sets I \hfill\label{eq:OCT2.0_assign}\\
&\sum_{f \in \sets F:x_f^i=1}b_{mf} \geq  \sum_{n\in \sets L: m\in \sets{ AR}(n)}\zeta_{a(n),n}^i &\hspace{-2cm}\forall i \in \sets I, m \in \sets N\label{eq:OCT2.0_right}\\
&\sum_{f \in \sets F:x_f^i=0}b_{mf} \geq  \sum_{n\in \sets L: m\in \sets{ AL}(n)}\zeta_{a(n),n}^i &\hspace{-2cm}\forall i \in \sets I, m \in \sets N\label{eq:OCT2.0_left}\\
&  \sum_{f \in \sets F}b_{nf} = 1  \quad  &\hspace{-2cm}\forall n \in \sets N \hfill\label{eq:OCT2.0_branch}\\
& z_{a(n),n}^i\leq  \zeta_{a(n),n}^i\quad &\hspace{-2cm}\forall n\in\sets L, i\in \sets I\\
&z_{a(n),n}^i\leq w_{ny^i}\quad &\hspace{-2cm}\forall n\in\sets L, i\in \sets I\\
&   \zeta_{a(n),n}^i \in \{0,1\}  \quad&\hspace{-2cm} \forall i \in \sets I, n \in \sets L \hfill\\
&   b_{nf}\in \{0,1\} \quad & \hspace{-10cm}\forall f \in \sets F, n \in \sets N.\label{eq:OCT2.0_last}
\end{align}
\end{subequations}

\subsection{Projection}
We now project out the $\zeta$ variables, obtaining a more compact formulation with the same LP relaxation. Specifically, consider the formulation 
\begin{subequations}\label{eq:Projection}
\begin{align}
\max \;\;& \displaystyle \sum_{i\in \sets I}\sum_{n\in \sets L}z_{n,a(n)}^i\\
\text{s.t. }\;\;& \sum_{k \in \sets K}w_{nk}=1 \quad &\hspace{-2cm}\forall n \in \sets L \hfill\\
&\sum_{f \in \sets F:x_f^i=1}b_{mf} \geq  \sum_{n\in \sets L: m\in \sets{ AR}(n)}z_{a(n),n}^i &\hspace{-2cm}\forall i \in \sets I, m \in \sets N\label{eq:Projection_right}\\
&\sum_{f \in \sets F:x_f^i=0}b_{mf} \geq  \sum_{n\in \sets L: m\in \sets{ AL}(n)}z_{a(n),n}^i &\hspace{-2cm}\forall i \in \sets I, m \in \sets N\label{eq:Projection_left}\\
&  \sum_{f \in \sets F}b_{nf} = 1  \quad  &\hspace{-2cm}\forall n \in \sets N \hfill\\
&z_{a(n),n}^i\leq w_{ny^i}\quad &\hspace{-2cm}\forall n\in\sets L, i\in \sets I\\
&   z_{a(n),n}^i \in \{0,1\}  \quad&\hspace{-2cm} \forall i \in \sets I, n \in \sets L \hfill\\
&   b_{nf}\in \{0,1\} \quad & \hspace{-10cm}\forall f \in \sets F, n \in \sets N..
\end{align}
\end{subequations}

\begin{proposition}
Formulations \eqref{eq:OCT2.0} and \eqref{eq:Projection} are equivalent, i.e., their LP relaxations have the same optimal objective value.
\end{proposition}
\begin{proof}
Let $\nu_1$ and $\nu_2$ be the optimal objective value of the LP relaxations of \eqref{eq:OCT2.0} and \eqref{eq:Projection}. Note that \eqref{eq:Projection} is a relaxation of \eqref{eq:OCT2.0}, obtained by dropping constraint \eqref{eq:OCT2.0_assign} and replacing $\zeta$ with a lower bound in constraints \eqref{eq:OCT2.0_right}-\eqref{eq:OCT2.0_left}. Therefore, it follows that $\nu_2\geq \nu_1$. We now show that $\nu_2\leq \nu_1$.

Let $(b^*,w^*,z^*)$ be an optimal solution of \eqref{eq:Projection} and let $i\in \sets I$. For any given $i\in \sets I$, by summing constraints \eqref{eq:Projection_right} and \eqref{eq:Projection_left} for the root node $m=1$, we find that 
\begin{align}\label{eq:z_ub}1=\sum_{f \in \sets F:x_f^i=1}b_{1f}^* +\sum_{f \in \sets F:x_f^i=0}b_{1f}^*
\geq  \sum_{n\in \sets L}(z_{a(n),n}^{i})^*.\end{align}
Now let $\zeta=z^*$. If the inequality in \eqref{eq:z_ub} holds at equality, then $(b^*,w^*,z^*,\zeta)$ satisfies all constraints in \eqref{eq:OCT2.0} and the proof is complete. Otherwise, it follows that either \eqref{eq:Projection_right} or \eqref{eq:Projection_left} is strict at node $m=1$, and without loss of generality assume \eqref{eq:Projection_right} is strict. Summing up inequalities \eqref{eq:Projection_right} and \eqref{eq:Projection_left} for node $m=r(1)$, we find that 
\begin{align}\label{eq:z_ub2}1=\sum_{f \in \sets F}b_{r(1)f}^*
>  \sum_{n\in \sets L:r(1)\in \sets{AR}(n)\cup \sets{AL}(n)}(z_{a(n),n}^{i})^*,\end{align}
where the strict inequality holds since the right hand side of \eqref{eq:z_ub2} is no greater than the right hand side of \eqref{eq:z_ub}. By applying this process recursively, we obtain a path from node 1 to a leaf $h\in \sets L$ such that all inequalities \eqref{eq:Projection_right}-\eqref{eq:Projection_left} corresponding to this path are strict. The value $\zeta_{a(h),h}^i$ can be then increased by the minimum slack in the constraints, and the overall process can be repeated until inequality \eqref{eq:OCT2.0_assign} is tight.
\end{proof}

\subsection{Substitution}
Finally, to recover the FlowOCT formulation, for all $m\in \sets L$, substitute variables 
\begin{align*}z_{m,r(m)}^i&:=\sum_{n\in \sets L: m\in \sets{ AR}(n)}z_{a(n),n}^i,\text{ and}\\ 
z_{m,\ell(m)}^i&:=\sum_{n\in \sets L: m\in \sets{ AL}(n)}z_{a(n),n}^i.
\end{align*}
and for all $n\in \sets L$ introduce variables $z_{n,t}=z_{a(n),n}$.
Constraints \eqref{eq:Projection_right}-\eqref{eq:Projection_left} reduce to $\sum_{f \in \sets F:x_f^i=1}b_{mf}\geq z_{m,r(m)}^i$ and $\sum_{f \in \sets F:x_f^i=0}b_{mf}\geq z_{m,\ell(m)}^i$. Finally, since 
\begin{align*}
    z_{a(m),m}&=\sum_{n\in \sets L:m\in \sets{AR}(n)\cup \sets{AL}(n)}z_{a(n),n}\\
    &=\sum_{n\in \sets L:m\in \sets{AR}(n)}z_{a(n),n}+\sum_{n\in \sets L:m\in \sets{AL}(n)}z_{a(n),n}\\
    &=z_{m,r(m)}+z_{m,\ell(m)},
\end{align*}
we recover the flow conservation constraints.

\section{Extended Results}
\label{appendix_sec:ext_results}

In Table~\ref{tab:in_sample}, for each dataset and depth, we show the in sample results for each approach. In this table, for $\lambda=0$, we average the in sample results including the training accuracy, optimality gap and solving time across five different samples trained over 50\% of the data when $\lambda$ is fixed to be $zero$.
Out of 32 instances, \texttt{OCT} has the best training accuracy in 0 instances (excluding ties) and \texttt{BinOCT} in 7 instances while \texttt{FlowOCT} and \texttt{Benders} have the best accuracy in 11 instances. 
In terms of solving time, \texttt{BinOCT} achieves a smaller solving time in 7 instances while \texttt{Benders} achieves a smaller solving time in 13 instances (excluding ties).
In terms of optimality gap, \texttt{OCT} achieves a smaller gap time only in one of the instances while \texttt{Benders} achieves a smaller gap time in 15 instances (excluding ties).

Similarly for $\lambda > 0$ we show similar results but this time for a given instance and a $\lambda \in [0.1,0.9]$ with step size of $0.1$ we solve 5 different samples and report the average results across all 45 samples. As \texttt{BinOCT} does not have any regularization term, we have excluded it from this section.
We observe that \texttt{Benders} outperform \texttt{OCT} in both optimality gap and solving time for all instances.

\begin{sidewaystable*}[t]
\caption{In sample results}
\label{tab:in_sample}
\setlength{\tabcolsep}{2pt}
\resizebox{\textheight}{!}{%
\begin{tabularx}{1.2\textwidth}{lc|ccc|ccc|ccc|ccc|cccccc}
\hline
\multirow{3}{*}{Dataset} &
  \multirow{3}{*}{Depth} &
  \multicolumn{12}{c|}{$\lambda = 0$} &
  \multicolumn{6}{c}{$\lambda > 0$} \\ \cline{3-20} 
 &
   &
  \multicolumn{3}{c|}{OCT} &
  \multicolumn{3}{c|}{BinOCT} &
  \multicolumn{3}{c|}{FlowOCT} &
  \multicolumn{3}{c|}{Benders} &
  \multicolumn{2}{c|}{OCT} &
  \multicolumn{2}{c|}{FlowOCT} &
  \multicolumn{2}{c}{Benders} \\
 &
   &
  Train-acc &
  Gap &
  Time &
  Train-acc &
  Gap &
  Time &
  Train-acc &
  Gap &
  Time &
  Train-acc &
  Gap &
  Time &
  Gap &
  \multicolumn{1}{c|}{Time} &
  Gap &
  \multicolumn{1}{c|}{Time} &
  Gap &
  Time \\ \hline
monk3 &
  2 &
  \textbf{94.0} &
  \textbf{0.0} &
  2.1 &
  \textbf{94.0} &
  \textbf{0.0} &
  \textbf{0.3} &
  \textbf{94.0} &
  \textbf{0.0} &
  6.7 &
  \textbf{94.0} &
  \textbf{0.0} &
  0.9 &
  \textbf{0.0} &
  \multicolumn{1}{c|}{2.8} &
  \textbf{0.0} &
  \multicolumn{1}{c|}{5.9} &
  \textbf{0.0} &
  \textbf{0.8} \\
monk3 &
  3 &
  \textbf{98.0} &
  \textbf{0.0} &
  283.6 &
  \textbf{98.0} &
  \textbf{0.0} &
  32.2 &
  \textbf{98.0} &
  \textbf{0.0} &
  89.1 &
  \textbf{98.0} &
  \textbf{0.0} &
  \textbf{16.1} &
  \textbf{0.0} &
  \multicolumn{1}{c|}{313.7} &
  \textbf{0.0} &
  \multicolumn{1}{c|}{23.7} &
  \textbf{0.0} &
  \textbf{4.0} \\
monk3 &
  4 &
  \textbf{100.0} &
  \textbf{0.0} &
  391.2 &
  \textbf{100.0} &
  \textbf{0.0} &
  654.3 &
  \textbf{100.0} &
  \textbf{0.0} &
  141.6 &
  \textbf{100.0} &
  \textbf{0.0} &
  \textbf{73.3} &
  0.5 &
  \multicolumn{1}{c|}{1197.3} &
  \textbf{0.0} &
  \multicolumn{1}{c|}{244.1} &
  \textbf{0.0} &
  \textbf{31.0} \\
monk3 &
  5 &
  \textbf{100.0} &
  \textbf{0.0} &
  203.8 &
  \textbf{100.0} &
  \textbf{0.0} &
  156.6 &
  \textbf{100.0} &
  \textbf{0.0} &
  117.7 &
  \textbf{100.0} &
  \textbf{0.0} &
  \textbf{5.9} &
  0.7 &
  \multicolumn{1}{c|}{1661.3} &
  \textbf{0.0} &
  \multicolumn{1}{c|}{262.5} &
  \textbf{0.0} &
  \textbf{35.2} \\
monk1 &
  2 &
  \textbf{85.8} &
  \textbf{0.0} &
  2.9 &
  \textbf{85.8} &
  \textbf{0.0} &
  1.5 &
  \textbf{85.8} &
  \textbf{0.0} &
  8.0 &
  \textbf{85.8} &
  \textbf{0.0} &
  \textbf{1.2} &
  \textbf{0.0} &
  \multicolumn{1}{c|}{4.0} &
  \textbf{0.0} &
  \multicolumn{1}{c|}{8.6} &
  \textbf{0.0} &
  \textbf{1.0} \\
monk1 &
  3 &
  \textbf{95.5} &
  \textbf{0.0} &
  672.1 &
  \textbf{95.5} &
  \textbf{0.0} &
  44.8 &
  \textbf{95.5} &
  \textbf{0.0} &
  45.7 &
  \textbf{95.5} &
  \textbf{0.0} &
  \textbf{4.8} &
  0.3 &
  \multicolumn{1}{c|}{1096.3} &
  \textbf{0.0} &
  \multicolumn{1}{c|}{25.6} &
  \textbf{0.0} &
  \textbf{4.3} \\
monk1 &
  4 &
  \textbf{100.0} &
  \textbf{0.0} &
  49.5 &
  \textbf{100.0} &
  \textbf{0.0} &
  61.9 &
  \textbf{100.0} &
  \textbf{0.0} &
  69.0 &
  \textbf{100.0} &
  \textbf{0.0} &
  \textbf{5.3} &
  0.3 &
  \multicolumn{1}{c|}{1042.1} &
  \textbf{0.0} &
  \multicolumn{1}{c|}{63.1} &
  \textbf{0.0} &
  \textbf{6.2} \\
monk1 &
  5 &
  \textbf{100.0} &
  \textbf{0.0} &
  116.9 &
  \textbf{100.0} &
  \textbf{0.0} &
  5.3 &
  \textbf{100.0} &
  \textbf{0.0} &
  214.3 &
  \textbf{100.0} &
  \textbf{0.0} &
  \textbf{2.7} &
  3.3 &
  \multicolumn{1}{c|}{2607.3} &
  \textbf{0.0} &
  \multicolumn{1}{c|}{154.5} &
  \textbf{0.0} &
  \textbf{8.4} \\
monk2 &
  2 &
  \textbf{71.4} &
  \textbf{0.0} &
  13.3 &
  \textbf{71.4} &
  \textbf{0.0} &
  \textbf{5.7} &
  \textbf{71.4} &
  \textbf{0.0} &
  12.8 &
  \textbf{71.4} &
  \textbf{0.0} &
  22.1 &
  \textbf{0.0} &
  \multicolumn{1}{c|}{15.6} &
  \textbf{0.0} &
  \multicolumn{1}{c|}{15.4} &
  \textbf{0.0} &
  \textbf{3.5} \\
monk2 &
  3 &
  81.0 &
  22.1 &
  3602.9 &
  80.5 &
  99.8 &
  3600.0 &
  \textbf{81.2} &
  \textbf{0.0} &
  1106.6 &
  \textbf{81.2} &
  \textbf{0.0} &
  \textbf{693.0} &
  18.5 &
  \multicolumn{1}{c|}{3076.1} &
  \textbf{0.0} &
  \multicolumn{1}{c|}{619.7} &
  \textbf{0.0} &
  \textbf{373.7} \\
monk2 &
  4 &
  88.3 &
  13.4 &
  3607.3 &
  86.7 &
  100.0 &
  \textbf{3600.0} &
  \textbf{89.5} &
  8.4 &
  3602.5 &
  \textbf{89.5} &
  \textbf{8.0} &
  \textbf{3600.0} &
  27.4 &
  \multicolumn{1}{c|}{3211.1} &
  11.0 &
  \multicolumn{1}{c|}{2886.9} &
  \textbf{9.2} &
  \textbf{2818.2} \\
monk2 &
  5 &
  92.6 &
  8.1 &
  3617.9 &
  94.8 &
  100.0 &
  3600.0 &
  93.6 &
  7.0 &
  3605.5 &
  \textbf{96.9} &
  \textbf{3.2} &
  \textbf{3505.4} &
  31.0 &
  \multicolumn{1}{c|}{3224.2} &
  14.4 &
  \multicolumn{1}{c|}{2987.8} &
  \textbf{13.1} &
  \textbf{2829.3} \\
house &
  2 &
  \textbf{97.1} &
  \textbf{0.0} &
  5.8 &
  \textbf{97.1} &
  \textbf{0.0} &
  \textbf{0.6} &
  \textbf{97.1} &
  \textbf{0.0} &
  12.9 &
  \textbf{97.1} &
  \textbf{0.0} &
  1.5 &
  \textbf{0.0} &
  \multicolumn{1}{c|}{4.3} &
  \textbf{0.0} &
  \multicolumn{1}{c|}{8.0} &
  \textbf{0.0} &
  \textbf{1.0} \\
house &
  3 &
  \textbf{99.0} &
  \textbf{0.0} &
  298.1 &
  \textbf{99.0} &
  \textbf{0.0} &
  180.4 &
  \textbf{99.0} &
  \textbf{0.0} &
  192.9 &
  \textbf{99.0} &
  \textbf{0.0} &
  \textbf{67.1} &
  \textbf{0.0} &
  \multicolumn{1}{c|}{407.5} &
  \textbf{0.0} &
  \multicolumn{1}{c|}{76.5} &
  \textbf{0.0} &
  \textbf{10.9} \\
house &
  4 &
  \textbf{100.0} &
  \textbf{0.0} &
  92.6 &
  99.8 &
  20.0 &
  1105.6 &
  \textbf{100.0} &
  \textbf{0.0} &
  276.9 &
  \textbf{100.0} &
  \textbf{0.0} &
  \textbf{13.4} &
  \textbf{0.0} &
  \multicolumn{1}{c|}{788.3} &
  \textbf{0.0} &
  \multicolumn{1}{c|}{164.6} &
  \textbf{0.0} &
  \textbf{22.2} \\
house &
  5 &
  \textbf{100.0} &
  \textbf{0.0} &
  83.1 &
  \textbf{100.0} &
  \textbf{0.0} &
  454.6 &
  \textbf{100.0} &
  \textbf{0.0} &
  206.9 &
  \textbf{100.0} &
  \textbf{0.0} &
  \textbf{18.8} &
  0.2 &
  \multicolumn{1}{c|}{1334.0} &
  \textbf{0.0} &
  \multicolumn{1}{c|}{310.7} &
  \textbf{0.0} &
  \textbf{25.7} \\
balance &
  2 &
  70.2 &
  \textbf{0.0} &
  159.4 &
  70.2 &
  \textbf{0.0} &
  \textbf{6.8} &
  \textbf{70.3} &
  \textbf{0.0} &
  49.5 &
  70.2 &
  \textbf{0.0} &
  7.0 &
  \textbf{0.0} &
  \multicolumn{1}{c|}{142.2} &
  \textbf{0.0} &
  \multicolumn{1}{c|}{92.1} &
  \textbf{0.0} &
  \textbf{6.4} \\
balance &
  3 &
  75.0 &
  33.6 &
  3613.4 &
  76.5 &
  95.3 &
  3600.0 &
  \textbf{76.6} &
  0.5 &
  3320.8 &
  \textbf{76.6} &
  \textbf{0.0} &
  \textbf{548.4} &
  37.8 &
  \multicolumn{1}{c|}{3613.4} &
  0.1 &
  \multicolumn{1}{c|}{2342.4} &
  \textbf{0.0} &
  \textbf{359.1} \\
balance &
  4 &
  75.5 &
  32.7 &
  3635.1 &
  78.8 &
  100.0 &
  \textbf{3600.0} &
  76.4 &
  30.5 &
  3611.3 &
  \textbf{79.7} &
  \textbf{9.7} &
  \textbf{3600.0} &
  43.2 &
  \multicolumn{1}{c|}{3635.0} &
  22.1 &
  \multicolumn{1}{c|}{3611.3} &
  \textbf{10.1} &
  \textbf{3340.2} \\
balance &
  5 &
  71.9 &
  39.1 &
  3686.4 &
  81.7 &
  100.0 &
  \textbf{3600.0} &
  78.8 &
  27.0 &
  3623.0 &
  \textbf{82.9} &
  \textbf{19.9} &
  \textbf{3600.0} &
  52.3 &
  \multicolumn{1}{c|}{3689.8} &
  29.3 &
  \multicolumn{1}{c|}{3622.8} &
  \textbf{24.4} &
  \textbf{3600.1} \\
tic-tac-toe &
  2 &
  72.7 &
  1.4 &
  2629.8 &
  72.7 &
  \textbf{0.0} &
  \textbf{89.7} &
  \textbf{72.8} &
  \textbf{0.0} &
  1198.7 &
  72.7 &
  \textbf{0.0} &
  118.3 &
  2.8 &
  \multicolumn{1}{c|}{2365.7} &
  \textbf{0.0} &
  \multicolumn{1}{c|}{1193.7} &
  \textbf{0.0} &
  \textbf{115.8} \\
tic-tac-toe &
  3 &
  77.5 &
  29.1 &
  3626.6 &
  \textbf{78.8} &
  100.0 &
  \textbf{3600.0} &
  76.0 &
  31.5 &
  3610.6 &
  78.7 &
  \textbf{18.8} &
  \textbf{3600.0} &
  33.6 &
  \multicolumn{1}{c|}{3626.5} &
  29.1 &
  \multicolumn{1}{c|}{3610.6} &
  \textbf{18.4} &
  \textbf{3600.2} \\
tic-tac-toe &
  4 &
  74.9 &
  33.6 &
  3670.0 &
  \textbf{85.0} &
  100.0 &
  \textbf{3600.0} &
  79.5 &
  26.0 &
  3623.1 &
  83.0 &
  \textbf{20.5} &
  \textbf{3600.0} &
  38.9 &
  \multicolumn{1}{c|}{3670.0} &
  27.2 &
  \multicolumn{1}{c|}{3621.7} &
  \textbf{24.8} &
  \textbf{3600.2} \\
tic-tac-toe &
  5 &
  71.4 &
  40.1 &
  3774.4 &
  \textbf{88.2} &
  100.0 &
  \textbf{3600.0} &
  70.4 &
  42.2 &
  3769.1 &
  \textbf{88.2} &
  \textbf{13.4} &
  \textbf{3600.0} &
  41.4 &
  \multicolumn{1}{c|}{3778.4} &
  35.1 &
  \multicolumn{1}{c|}{3650.6} &
  \textbf{24.9} &
  \textbf{3600.2} \\
car\_eval &
  2 &
  \textbf{78.5} &
  \textbf{0.0} &
  1047.8 &
  \textbf{78.5} &
  0.0 &
  \textbf{27.7} &
  76.9 &
  7.7 &
  1214.8 &
  \textbf{78.5} &
  \textbf{0.0} &
  66.4 &
  \textbf{0.0} &
  \multicolumn{1}{c|}{1610.7} &
  0.8 &
  \multicolumn{1}{c|}{679.6} &
  \textbf{0.0} &
  \textbf{55.1} \\
car\_eval &
  3 &
  77.2 &
  29.9 &
  3634.7 &
  \textbf{81.9} &
  100.0 &
  \textbf{3600.0} &
  79.4 &
  26.2 &
  3614.9 &
  81.7 &
  \textbf{9.3} &
  \textbf{3600.0} &
  36.0 &
  \multicolumn{1}{c|}{3634.8} &
  22.2 &
  \multicolumn{1}{c|}{3614.9} &
  \textbf{6.7} &
  \textbf{3571.2} \\
car\_eval &
  4 &
  76.5 &
  31.0 &
  3692.3 &
  \textbf{83.9} &
  100.0 &
  \textbf{3600.0} &
  71.9 &
  39.1 &
  3630.8 &
  83.6 &
  \textbf{19.6} &
  \textbf{3600.0} &
  37.9 &
  \multicolumn{1}{c|}{3691.9} &
  28.2 &
  \multicolumn{1}{c|}{3630.6} &
  \textbf{20.1} &
  \textbf{3600.1} \\
car\_eval &
  5 &
  72.1 &
  39.1 &
  3828.6 &
  85.0 &
  100.0 &
  \textbf{3600.0} &
  71.6 &
  39.9 &
  3742.6 &
  \textbf{85.2} &
  \textbf{17.3} &
  \textbf{3600.0} &
  37.3 &
  \multicolumn{1}{c|}{3828.4} &
  162072.3 &
  \multicolumn{1}{c|}{3753.2} &
  \textbf{18.9} &
  \textbf{3600.1} \\
kr-vs-kp &
  2 &
  84.7 &
  18.0 &
  3641.1 &
  \textbf{86.9} &
  \textbf{0.0} &
  \textbf{102.6} &
  82.1 &
  14.3 &
  3481.7 &
  \textbf{86.9} &
  \textbf{0.0} &
  671.2 &
  24.4 &
  \multicolumn{1}{c|}{3641.2} &
  18.4 &
  \multicolumn{1}{c|}{3544.5} &
  \textbf{0.0} &
  \textbf{652.1} \\
kr-vs-kp &
  3 &
  73.8 &
  35.6 &
  3725.7 &
  \textbf{92.6} &
  100.0 &
  \textbf{3600.0} &
  61.1 &
  65.0 &
  3701.0 &
  90.4 &
  \textbf{10.7} &
  \textbf{3600.0} &
  40.1 &
  \multicolumn{1}{c|}{3721.9} &
  49.2 &
  \multicolumn{1}{c|}{3671.3} &
  \textbf{11.7} &
  \textbf{3600.2} \\
kr-vs-kp &
  4 &
  68.1 &
  47.2 &
  3923.1 &
  \textbf{92.5} &
  100.0 &
  \textbf{3600.0} &
  63.8 &
  58.1 &
  4560.5 &
  89.8 &
  \textbf{11.8} &
  \textbf{3600.0} &
  49.6 &
  \multicolumn{1}{c|}{3923.7} &
  419947.6 &
  \multicolumn{1}{c|}{4479.5} &
  \textbf{9.7} &
  \textbf{3600.2} \\
kr-vs-kp &
  5 &
  66.6 &
  50.9 &
  4398.2 &
  \textbf{93.6} &
  100.0 &
  \textbf{3600.0} &
  51.5 &
  10000 &
  3790.7 &
  90.5 &
  \textbf{10.7} &
  \textbf{3600.0} &
  50.5 &
  \multicolumn{1}{c|}{4400.7} &
  677022.8 &
  \multicolumn{1}{c|}{4789.0} &
  \textbf{12.5} &
  \textbf{3600.3} \\ \hline
\end{tabularx}}
\end{sidewaystable*}

\end{document}